%% file: arxiv_main.tex
\documentclass[11pt]{article}
\usepackage{graphicx} % Required for inserting images

\usepackage{natbib}
\usepackage{times}
\usepackage{amsmath}
\usepackage{amssymb}
\usepackage{amsthm}
\usepackage{xcolor}
\usepackage[ruled,vlined]{algorithm2e}
\usepackage{thm-restate}
\usepackage{mathtools}
\usepackage{thmtools}
\usepackage{enumitem}
\usepackage{circledsteps}
\usepackage{subcaption}
\usepackage{booktabs}
\usepackage{tikz}
\usetikzlibrary{arrows.meta,positioning,calc}
\usepackage{ifthen}
\newboolean{isSingleColumn}

\input{commands}

\DeclareMathOperator*{\argmax}{arg\,max}
\DeclareMathOperator*{\argmin}{arg\,min}

\newtheorem{theorem}{Theorem}
\newtheorem{lemma}[theorem]{Lemma}
\newtheorem{corollary}[theorem]{Corollary}
\newtheorem{definition}[theorem]{Definition}
\newtheorem{proposition}[theorem]{Proposition}

\makeatletter
\if@twocolumn
    \setboolean{isSingleColumn}{false}
\else
    \setboolean{isSingleColumn}{true}
\fi

\newcommand{\adaptiveHeader}[1]{%
    \ifthenelse{\boolean{isSingleColumn}}{%
        \paragraph{#1} % Use paragraph style in single column
    }{%
        \textbf{#1}    % Use bold text in double column
    }%
}
\newcommand{\imgwidth}{\linewidth}
\newcommand{\tabwidth}{\linewidth}

\usepackage{fullpage}
\usepackage{authblk}

\usepackage{hyperref}
\usepackage{cleveref}

\usepackage[most]{tcolorbox}
\tcbset{
  keyresult/.style={
    colback=gray!15,
    colframe=white,
    boxrule=0pt,
    enhanced,
    arc=2pt,
    left=6pt,
    right=6pt,
    top=3pt,
    bottom=3pt,
    before skip=6pt,
    after skip=6pt,
    fontupper=\normalfont
  }
}

\title{Preference Elicitation for Policy Optimization \\ and Application to Aligning Heart Transplantation with Human Values}

\author[1]{Itai Zilberstein\thanks{Correspondence to \texttt{izilbers@cs.cmu.edu, sandholm@cs.cmu.edu}.}}
\author[1]{Ioannis Anagnostides}
\author[2]{Zachary W. Sollie}
\author[2]{Arman Kilic}
\author[1,3]{Tuomas Sandholm}

\affil[1]{Department of Computer Science, Carnegie Mellon University, Pittsburgh, PA}
\affil[2]{Department of Surgery, Division of Cardiothoracic Surgery, Medical University of South Carolina, Charleston, SC}
\affil[3]{\small{Additional affiliations: Strategy Robot, Inc., Strategic Machine, Inc., Optimized Markets, Inc.}}

\begin{document}

\maketitle
\thispagestyle{empty}

\begin{abstract}
    Preference elicitation is essential for aligning AI systems with human values. Prior approaches (\textit{e.g.}, for organ allocation) often ask stakeholders to compare the decisions of an algorithm (\textit{e.g.}, patient A vs. patient B). Such a decision-level approach conflates the \emph{means} with the \emph{ends}. Instead, we elicit preferences directly over allocation outcomes to learn a utility function for policy optimization. We construct a novel preference elicitation algorithm for linear utilities that outperforms prior techniques in practice. Our algorithm has two phases. The first phase learns cutting planes through pairwise comparisons to rapidly shrink the space of possible attribute weights and warm-starts the second phase by eliminating dominated regions. The second phase then provably converges to the user’s utility function. We apply our technique to heart transplant allocation, where a policy must balance competing objectives such as post-transplant outcomes, waitlist mortality, geographic ease, and equity. Using our algorithm, we conduct a user study to learn and aggregate a human-value-aligned utility function. We evaluate the linear representation against LLM-based methods supplied with users’ natural language explanations, finding that none of these methods is more accurate than the learned linear utilities. We use the resulting utility function to optimize heart transplant policies that are significantly better aligned with human values. Compared to the hindsight optimum, the \textit{status quo} policy achieves a competitive ratio of just $0.66$, while our method is near-optimal with a competitive ratio of $0.95$.
\end{abstract}

\clearpage
\setcounter{page}{1}

\input{text/introduction_v2}
\input{text/perliminaries}
\input{text/algorithms}
\input{text/heart_transplant}
\input{text/conclusions}
\input{text/acks}

\bibliography{refs}

\clearpage

\appendix
\input{text/appendix}

\end{document}

%% file: commands.tex
\newcommand{\E}{\mathbb{E}}

\newcommand{\R}{\mathbb{R}}
\newcommand{\Sp}{\mathbb{S}}
\newcommand{\Objset}{\Sp^{d-1}_+}

\newcommand{\obj}{d}
\newcommand{\vx}{\mathbf{x}}
\newcommand{\vw}{\mathbf{w}}
\newcommand{\vy}{\mathbf{y}}
\newcommand{\vu}{\mathbf{u}}
\newcommand{\vb}{\mathbf{b}}
\newcommand{\vz}{\mathbf{z}}
\newcommand{\va}{\mathbf{a}}
\DeclareMathOperator{\diam}{diam}

\newcommand{\vd}{\mathbf{d}}
\newcommand{\hvw}{\widehat{\vw}}

\newcommand{\ut}{u}
\newcommand{\inner}[2]{\left\langle {#1}, {#2} \right\rangle}
\newcommand{\norm}[1]{\left\lVert {#1} \right\rVert}

\newcommand\envelope{\operatorname{env}}

\newcommand{\sse}{\subseteq}
\newcommand{\ip}[1]{\langle #1 \rangle}
\newcommand\cone{\operatorname{cone}}

\newcommand{\cF}{\mathcal{F}}

\newcommand{\cK}{\mathcal{K}}

\newcommand{\cP}{\mathcal{P}}

\newcommand{\cS}{\mathcal{S}}

\newcommand{\cU}{\mathcal{U}}

\newcommand{\cW}{\mathcal{W}}
\newcommand{\cX}{\mathcal{X}}
\newcommand{\cY}{\mathcal{Y}}

\newcommand{\eps}{\epsilon}
\newcommand{\BB}{\mathbb{B}}
\newcommand{\CDfull}{conic dimension\xspace}
\newcommand{\CDmath}{\operatorname{ConicDim}}
\newcommand{\CDshort}{conic dimension\xspace}
\newcommand{\dist}{\operatorname{dist}}

%% file: text/introduction_v2.tex
\section{Introduction}

% Start with preference elicitation and then go into application. Credit Dickerson and Sandholm for means and ends separation for kidney exchange in AAAI'15. Rename version space.  

Preference elicitation is a fundamental capability for AI systems that make decisions on behalf of humans. Many decision-support systems require balancing multiple competing objectives, yet the relative importance of these objectives is often difficult to specify \textit{a priori}. Preference elicitation addresses this challenge by learning a user’s utility function through interactive queries that are selected based on the user's responses so far. The elicited preferences alone, however, are not the end goal: ultimately, an AI system must make decisions on behalf of humans. As AI is increasingly deployed in high-stakes domains, there is a growing need for methodologies that bridge preference elicitation and downstream optimization so that AI systems make decisions aligned with human values.

%Stakeholders should specify \emph{what} a decision-making system should achieve while optimization determines \emph{how} best to achieve it.

% As AI is increasingly deployed in high-stakes domains, there is a growing need for preference elicitation algorithms that are both query-efficient in practice and supported by theoretical guarantees that the learned preferences are accurate.

This distinction is especially important in organ allocation, one of the highest-stakes algorithmic problems in healthcare. For patients suffering from end-stage organ failure, transplantation is the leading treatment. However, across organ types, the demand for donor organs far exceeds the available supply. In the US, over 100,000 candidates are on waiting lists for solid-organ transplants~\citep{hrsa26:Stats}. 

As organ allocation shifts towards data-driven algorithms to improve patient outcomes, a crucial challenge is to balance the competing objectives of the system~\citep{OPTN_Ethics,Keswani25:Can}. 
Organ allocation is inherently multi-objective, requiring tradeoffs among post-transplant outcomes, waitlist mortality, geographic ease, equity, and other ethical and operational considerations.

\ifthenelse{\boolean{isSingleColumn}}{%
        \renewcommand{\imgwidth}{0.75\linewidth}
    }{%
        \renewcommand{\imgwidth}{0.75\linewidth}
    }%
    
\begin{figure*}[b!]
\centering
\resizebox{\imgwidth}{!}{%
    \input{figures/pipeline_figure_v2}
}
\caption{
Our human-aligned policy optimization framework. We use preference elicitation to learn \emph{what} stakeholders want the allocation system to achieve; optimization determines \emph{how} best to achieve it.
}
\label{fig:elicitation-policy-pipeline}
\end{figure*}

Existing preference elicitation approaches commonly ask stakeholders to compare patient attributes or individual allocation decisions, such as whether one patient should be prioritized over another~\citep{Freedman20:Adapting,Cummiskey25:Understanding,Dickerson26:Who}. Such questions ask stakeholders to reason directly about the \emph{mechanism} used to allocate organs, conflating the decisions of an algorithm with the outcomes of the system. This decision-level approach was adopted to inform \emph{continuous distribution}~\cite{Papalexopoulos24:Reshaping}, an allocation policy recently deployed for lungs and under consideration for hearts in the US~\citep{OPTN_Heart_CD}. 

We argue that existing preference elicitation for organ allocation typically conflates the \textit{means} with the \textit{ends}, a distinction originally introduced by \citet{Dickerson15:Futurematch} for kidney exchange. This distinction, however, was only introduced for policy optimization, not preference elicitation. Decision-level approaches do not directly capture what stakeholders ultimately want the system to achieve and do not yield an objective for computationally optimizing a policy. We argue---and justify by the results in this paper---that rather than eliciting preferences over individual allocation decisions or scoring-rule parameters, stakeholders should express preferences over the long-run goals of the system, while optimization methods should determine how policies can best achieve those goals. 

\begin{tcolorbox}[keyresult]
We formulate outcome-level preference elicitation to directly perform downstream policy optimization.
\end{tcolorbox}

We bridge the gap between practical preference elicitation and direct policy optimization by eliciting a utility over system-level outcomes. This formulation creates an algorithmic challenge: the elicitation procedure must be query-efficient enough for human stakeholders while providing a formal guarantee that the recovered utility is accurate. The framework should also yield a policy that is near-optimal for the elicited preferences. We present a framework that meets these desiderata, depicted in \Cref{fig:elicitation-policy-pipeline}.

\subsection{Our contributions}

We first study the query complexity of preference elicitation. We show that optimizing arbitrary monotone utilities, a weaker task than fully learning them, requires an exponential number of pairwise queries in the worst case (\Cref{thm:general-utility-lower-bound}), motivating our focus on linear utilities. Linear utility models are the canonical representation for multiattribute preferences~\citep{Fishburn70:Utility,Keeney93:Decisions}. They also remain one of the most widely studied representations across preference elicitation, mechanism design, and decision support (\textit{e.g.},~\citep{Chajewska00:Making,Boutilier04:Eliciting,Parkes05:Optimize,Conitzer07:Eliciting,Cummiskey25:Understanding,Vayanos26:Robust}). 

We then turn to practical methods for preference elicitation of a linear utility function. We construct a novel preference elicitation algorithm that provably converges to the user's utility function and outperforms prior techniques in practice. Our algorithm has two phases. The first phase learns cutting planes through pairwise comparisons to rapidly shrink the space of possible attribute weights. It warm starts the second phase by eliminating all dominated portions of the weight space. The second phase converges to the user's utility function at a provable rate through an iterative sieving procedure. In simulation, our algorithm achieves a target approximation using fewer queries than the cutting plane or sieving method. Our algorithm also beats the other algorithms from the literature, and is on par with the recent algorithm by \citet{Ge24:Learning}, which has a near-optimal query complexity of $O(d\log(d/\epsilon))$. Our algorithm, however, demonstrates much better robustness to noise than it, achieving $30\%$ less error.

%Their algorithm performs binary search on each coordinate dimension resulting in asking extreme and impractical queries of the user. It also assumes noise-free answers (their noise-robust implementation asks the same query multiple times). Our algorithm requires roughly the same number of queries while incurring $30\%$ less error in noisy settings and asking more practical queries.

We apply our preference elicitation algorithm to heart transplant allocation to learn a human-value-aligned utility function. Our elicitation algorithm requires roughly 40 queries per participant on average. In addition to pairwise responses, users optionally supply natural language explanations of their decision making. We evaluate \textit{large language model (LLM)}-based predictors using their elicitation responses and explanations, with and without guidance from the learned linear utility. None of the evaluated LLM-based methods improve accuracy over the linear model, providing empirical support for the linear representation in our setting. Finally, we feed the aggregated linear utility function as input to a policy optimization algorithm to learn an allocation policy that optimizes the human-aligned utility function. We show that our optimized policy dominates the current US \textit{status quo} allocation policy across nearly all objectives.  

\begin{tcolorbox}[keyresult]
Compared to the hindsight optimum under the elicited stakeholder utility, the \textit{status quo} policy achieves a competitive ratio of just $\mathbf{0.66}$, while our optimized policy is near-optimal with a competitive ratio of $\mathbf{0.95}$.
\end{tcolorbox}

The current heart transplant allocation system is under review in the US with the goal to better align outcomes with human values. Prior work on policy optimization for organ allocation takes the utility function as a given; we close this loop by directly optimizing the utility function elicited from humans. Our work comes at a pivotal moment for US heart-allocation policy and offers a framework for the development and evaluation of future allocation policies that incorporate stakeholder values. More broadly, we provide a practically deployable pipeline for learning what a policy should achieve and then optimizing the policy to achieve it.

\subsection{Related work}

The historical mechanism for allocating organs has relied on explicit rule-based policies developed by domain experts. More recently, as data availability has increased and machine learning techniques have improved, there is a push to leverage data-driven approaches to optimize organ allocation policies. In the US and abroad, simulation and optimization have already been used to deploy new rules~\citep{Kamath01:Model,Gottlieb17:Lung,Allen24:Transplant,OPTN25:ContinuousDistribution}.

\emph{Continuous distribution}, an allocation framework recently adopted for lungs and under consideration for heart transplantation in the US, prioritizes patients using a \textit{composite allocation score (CAS)} which aggregates patient attributes through a linear weighted scoring rule~\citep{Cummiskey25:Understanding}. Current approaches determine these weights using the \emph{analytic hierarchy process (AHP,~\citep{Saaty77:Scaling})}, which elicits stakeholder preferences over patient attributes. Similar decision-level preference elicitation has also been studied in kidney exchange to identify which of two patients should be prioritized~\citep{Freedman20:Adapting,Dickerson26:Who}. 

Preference elicitation has been studied extensively as a way to make high-quality decisions without requiring users to fully specify their utility functions~\citep{Chajewska00:Making,Blum04:Preference}. Early work studied the query complexity of eliciting preferences through rich query models, including value queries~\citep{Zinkevich03:Polynomial}. To impose low cognitive burden on users, pairwise comparisons are attractive compared to numerical specification~\citep{Chajewska00:Making,Conitzer07:Eliciting}.

Our work closely relates to active learning from comparison queries. Comparison-based learning has been studied from both practical and theoretical perspectives~\citep{Kane17:Active,Eric07:Active,Sadigh17:Active,Biyik18:Batch,Johnston23:Deploying}.

Eliciting preferences on policy outcomes rather than decisions has been examined indirectly in other applications to perform policy \textit{selection} from a fixed set of contender policies~\citep{Yu20:Keeping,Vayanos26:Robust}. Our approach is inverted from those prior approaches: instead of starting from a fixed set of policies, we learn a general utility function and then \textit{optimize} a policy according to it. 

The distinction between policy outcomes and decisions has also been examined through the lens of preference aggregation~\citep{Conitzer15:Crowdsourcing,Conitzer16:Rules,Zhang19:Better}. Their aggregation rules could be adapted in future work to our setting to ensure that the final utility function satisfies properties other than maximizing welfare, which is what we focus on, and obtain, in this paper. 

Another key benefit of our approach to preference elicitation is that the learned utility function is compatible with existing optimization-based methods for donor-patient matching~\citep{Su04:Patient,Abraham07:Clearing,Awasthi09:Online,Dickerson12:Dynamic,Dickerson15:Futurematch,Dickerson16:Position,Berrevoets20:OrganITE,Berrevoets21:Learning,Anagnostides25:Policy,Zilberstein26:Near,Zilberstein26:Aligning,Zilberstein26:Learning}.
We provide further pointers to related work in~\Cref{sec:related}.

%% file: figures/pipeline_figure_v2.tex
\begin{tikzpicture}[
    font=\footnotesize,
    node distance=1.5cm,
    pipelinebox/.style={
        draw=black!45,
        rounded corners=2pt,
        minimum width=3.25cm,
        minimum height=2.75cm,
        text width=2.45cm,
        align=center,
        inner sep=5pt
    },
    pipelinetitle/.style={
        font=\footnotesize\bfseries,
        text width=2.25cm,
        align=center
    },
    pipelinearrow/.style={
        -{Latex[length=2.5mm,width=2.5mm]},
        line width=1.2pt,
        draw=black
    },
    stepnumber/.style={
        circle,
        fill=blue!75!black,
        text=white,
        font=\bfseries,
        minimum size=5.5mm,
        inner sep=0pt
    },
    description/.style={
        font=\scriptsize,
        text width=2.25cm,
        align=center
    }
]

% ================================================================
% STEP 1: Elicit individual utility functions
% ================================================================
\node[pipelinebox] (step1) {};

\node[stepnumber]
    at ([xshift=2mm,yshift=-2mm]step1.north west)
    {1};

\node[pipelinetitle, anchor=north]
    at ([yshift=-3mm]step1.north)
    {Elicit individual utility functions};

\begin{scope}[shift={([yshift=-20mm]step1.north)}]

    % Back card: User 3
    \draw[
        fill=white,
        draw=black!35,
        rounded corners=1pt
    ]
        (-1.15,-0.02) rectangle (-0.30,0.48);

    \node[
        font=\scriptsize\bfseries
    ]
        at (-0.725,0.23)
        {User 3};

    % Middle card: User 2
    \draw[
        fill=white,
        draw=black!42,
        rounded corners=1pt
    ]
        (-1.25,-0.27) rectangle (-0.40,0.23);

    \node[
        font=\scriptsize\bfseries
    ]
        at (-0.825,-0.02)
        {User 2};

    % Front card: User 1
    \draw[
        fill=white,
        draw=black!50,
        rounded corners=1pt
    ]
        (-1.35,-0.52) rectangle (-0.50,-0.02);

    \node[
        font=\scriptsize\bfseries
    ]
        at (-0.925,-0.27)
        {User 1};

    % Arrows from users to learned utilities
    \draw[
        -{Latex[length=1.7mm,width=1.4mm]},
        black!60,
        line width=0.6pt
    ]
        (-0.28,0.23) -- (0.18,0.23);

    \draw[
        -{Latex[length=1.7mm,width=1.4mm]},
        black!60,
        line width=0.6pt
    ]
        (-0.38,-0.02) -- (0.18,-0.02);

    \draw[
        -{Latex[length=1.7mm,width=1.4mm]},
        black!60,
        line width=0.6pt
    ]
        (-0.48,-0.27) -- (0.18,-0.27);

    % Individual utility functions
    \node[
        blue!75!black,
        font=\scriptsize\bfseries,
        anchor=west
    ]
        at (0.22,0.23)
        {$\widehat{\vu}^{(3)}$};

    \node[
        blue!75!black,
        font=\scriptsize\bfseries,
        anchor=west
    ]
        at (0.22,-0.02)
        {$\widehat{\vu}^{(2)}$};

    \node[
        blue!75!black,
        font=\scriptsize\bfseries,
        anchor=west
    ]
        at (0.22,-0.27)
        {$\widehat{\vu}^{(1)}$};

\end{scope}

\node[description, anchor=south]
    at ([yshift=3mm]step1.south)
    {};

% ================================================================
% STEP 2: Aggregate stakeholder utilities
% ================================================================
\node[pipelinebox, right=of step1] (step2) {};

\node[stepnumber]
    at ([xshift=2mm,yshift=-2mm]step2.north west)
    {2};

\node[pipelinetitle, anchor=north]
    at ([yshift=-3mm]step2.north)
    {Aggregate utility functions};

% Coordinate axes and aggregate utility direction
\begin{scope}[shift={([yshift=-20mm]step2.north)}]

    \draw[-{Latex[length=1.7mm]}, black!65]
        (-0.75,-0.48) -- (0.85,-0.48);

    \draw[-{Latex[length=1.7mm]}, black!65]
        (-0.75,-0.48) -- (-0.75,0.67);

    \draw[dashed, black!35]
        (-0.75,-0.48) -- (0.52,-0.05);

    \draw[
        -{Latex[length=2.4mm]},
        blue!75!black,
        very thick
    ]
        (-0.75,-0.48) -- (0.56,0.50);

    \node[
        blue!75!black,
        font=\small\bfseries
    ]
        at (0.68,0.56)
        {$\bar{\vu}$};

\end{scope}

\node[description, anchor=south]
    at ([yshift=3mm]step2.south)
    {};

% ================================================================
% STEP 3: Optimize policy
% ================================================================
\node[pipelinebox, right=of step2] (step3) {};

\node[stepnumber]
    at ([xshift=2mm,yshift=-2mm]step3.north west)
    {3};

\node[pipelinetitle, anchor=north]
    at ([yshift=-3mm]step3.north)
    {Optimize an allocation policy};

% Feasible policies represented as points
\begin{scope}[shift={([yshift=-20mm]step3.north)}]

    \foreach \x/\y in {
        -0.72/-0.30,
        -0.53/0.15,
        -0.30/-0.06,
        -0.06/0.31,
        0.18/-0.27,
        0.42/0.03,
        0.66/0.40
    }{
        \fill[green!55!black]
            (\x,\y) circle (1.7pt);
    }

    \draw[
        -{Latex[length=2.4mm]},
        blue!75!black,
        very thick
    ]
        (-0.72,-0.48) -- (0.67,0.55);

    \node[
        blue!75!black,
        font=\scriptsize\bfseries,
        fill=white,
        inner sep=1pt
    ]
        at (0.10,0.57)
        {maximize $\bar{\vu}$};

\end{scope}

\node[description, anchor=south]
    at ([yshift=3mm]step3.south)
    {};

% ================================================================
% STEP 4: Human-aligned outcomes
% ================================================================
\node[pipelinebox, right=of step3] (step4) {};

\node[stepnumber]
    at ([xshift=2mm,yshift=-2mm]step4.north west)
    {4};

\node[pipelinetitle, anchor=north]
    at ([yshift=-3mm]step4.north)
    {Human-aligned outcomes};

% Simplified five-axis spider plot
\begin{scope}[shift={([yshift=-20mm]step4.north)}]

    \def\radarR{0.78}

    % Five radial axes
    \foreach \angle in {90,18,-54,-126,162}{
        \draw[black!22, line width=0.45pt]
            (0,0) --
            ({\radarR*cos(\angle)},
             {\radarR*sin(\angle)});
    }

    % Nested pentagonal grid
    \foreach \scale in {0.25,0.50,0.75,1.00}{
        \draw[black!22, line width=0.45pt]
            ({\scale*\radarR*cos(90)},
             {\scale*\radarR*sin(90)})
        \foreach \angle in {18,-54,-126,162}{
            -- ({\scale*\radarR*cos(\angle)},
                {\scale*\radarR*sin(\angle)})
        }
        -- cycle;
    }

    % Achieved outcome profile: outline only
    \draw[
        green!55!black,
        very thick
    ]
        ({0.74*\radarR*cos(90)},
         {0.74*\radarR*sin(90)})
        --
        ({0.70*\radarR*cos(18)},
         {0.70*\radarR*sin(18)})
        --
        ({0.83*\radarR*cos(-54)},
         {0.83*\radarR*sin(-54)})
        --
        ({0.76*\radarR*cos(-126)},
         {0.76*\radarR*sin(-126)})
        --
        ({0.68*\radarR*cos(162)},
         {0.68*\radarR*sin(162)})
        -- cycle;

    % Outcome markers
    \foreach \angle/\scale in {
        90/0.74,
        18/0.70,
        -54/0.83,
        -126/0.76,
        162/0.68
    }{
        \fill[green!55!black]
            ({\scale*\radarR*cos(\angle)},
             {\scale*\radarR*sin(\angle)})
            circle (1.6pt);
    }

\end{scope}

\node[description, anchor=south]
    at ([yshift=3mm]step4.south)
    {};

% ================================================================
% Pipeline arrows
% ================================================================
\draw[pipelinearrow]
    (step1.east) -- (step2.west);

\draw[pipelinearrow]
    (step2.east) -- (step3.west);

\draw[pipelinearrow]
    (step3.east) -- (step4.west);

\end{tikzpicture}

%% file: text/perliminaries.tex
\section{Utility functions}
Let $\vx\in \R^\obj_{\ge 0}$ be a $\obj$-dimensional outcome vector. We assume each stakeholder has a monotone utility function $\ut : \R^\obj_{\ge0}\to\R$. For a set $\cX$, we use the notation $\diam_2(\cX)$ to denote its $\ell_2$ diameter and define the distance to a set as $\dist(\vx,\cX) = \min_{\vy \in \cX} \norm{\vx-\vy}_2$. The goal is to learn a utility function that captures a user's preferences over a set of points, $\cP$, using pairwise comparison queries. 

In heart transplant allocation, each coordinate of $\vx$ corresponds to an aggregate system-level objective of an allocation policy rather than a patient-level attribute. The set $\cP \subseteq \mathbb R^d_{\ge 0}$ represents the feasible set of \textit{outcome} profiles induced by allocation policies. The set may change as the policy class, patient population, donor supply, or clinical practice evolve. We therefore distinguish the downstream policy set $\cP$ from the domain on which preferences are elicited. We elicit preferences over a frontier-agnostic query domain $\cX= \Objset = \{ \vx \in \R^\obj_{\ge 0} \mid \norm{\vx}_2  = 1 \}$. Points in $\cX$ represent hypothetical normalized outcome profiles and need not correspond to a currently feasible allocation policy. A learned utility function can then be evaluated on a current or future feasible policy set $\cP$. 

\subsection{A motivating lower bound}
\label{sec:motivation-lower-bound}

The utility functions considered above could, in principle, be arbitrary monotone functions over the objective space. Without additional structure, eliciting enough information to optimize such utilities is intractable. The following proposition shows that welfare maximization even over just two arbitrary monotone utility functions  requires exponentially many bits of information in the number of objectives. 

\begin{restatable}[Exponential lower bound for arbitrary monotone utilities]{proposition}{GeneralUtilityLowerBound}
\label{thm:general-utility-lower-bound}
For every even $\obj$, any protocol that, for every pair of coordinate-wise monotone utilities $\ut_1,\ut_2:\R^\obj_{\ge0}\to\{0,1\}$, 
outputs an exact maximizer $\widehat{\vx} \in \argmax_{\vx\in\Objset} \left(\ut_1(\vx)+\ut_2(\vx)\right)$ with probability at least $2/3$ requires $2^{\Omega(\obj)}$ bits of communication in the worst case.
\end{restatable}

The proof of~\Cref{thm:general-utility-lower-bound}, along with all other omitted proofs, appears in~\Cref{sec:appendix-proofs} of the Appendix.

Learning arbitrary utility functions well enough to support downstream optimization is at least as hard as finding a welfare-maximizing outcome, since an exact representation of the utilities would imply the ability to solve the welfare-maximization problem. Therefore, the result above motivates our focus on structured utility classes, and we will focus on the most prevalent utility function class in the literature, namely linear utilities.

\subsection{Linear utility functions and normalization}

A \textit{linear utility function} is specified by a weight vector $\vw \in \R^\obj_{\ge 0}$ and $\ut(\vx) = \inner{\vx}{\vw}$.  Linear utility functions are invariant to positive rescaling: for any $\lambda>0$, the weights $\vw$ and $\lambda\vw$ induce the same preference ordering over outcomes. We therefore identify a linear utility function by its positive ray and fix a canonical representative by normalizing weights. We use two equivalent normalizations of this preference direction. The simplex-normalized representative is denoted by $\vw\in\cW$, where
$ \cW = \left\{\vw\in\R^\obj_{\ge 0} \mid \norm{\vw}_1=1 \right\}$. When Euclidean normalization is more convenient, we represent the same preference direction by a vector $\vu\in\Objset$. We support both normalizations because the different preference elicitation algorithms, which we use as components of our elicitation algorithm, are designed for one or the other. These two normalizations have the same underlying preference direction and can be bijectively mapped to one another. Any algorithm that computes an $\epsilon$-approximation in one normalization can be used to generate an $\epsilon'$-approximation in the other normalization, where $\epsilon' = \epsilon \sqrt{d}$. 

%The simplex normalization is consistent with the continuous distribution framework in which attributes are aggregated through normalized weights~\citep{OPTN25:ContinuousDistribution}.

%% file: text/algorithms.tex
\section{Eliciting linear preferences}

In this section, we present three methods for learning a user’s linear utility function using pairwise comparisons. The first method uses a cutting plane approach and operates over the continuous domain $\Objset$. The cutting plane approach performs well in practice, but does not have any \textit{ex ante} finite-query convergence or solution-quality guarantees. Second, in the subsequent subsection, we adapt the recent iterative sieving algorithm of~\citet{Cohen25:Combinatorial} to construct an elicitation algorithm that provably converges to the user's utility function. Third, we present our hybrid algorithm. It combines the practical  efficiency of the cutting plane method with the provable convergence of the sieving algorithm, and empirically outperforms both of them---and the other prior elicitation algorithms, which we will describe later. 

We assume that the user has a utility function specified by a latent, true weight vector $\vw^* \in \cW$, which governs the underlying tradeoffs. We seek to learn a weight vector $\widehat{\vw}$ that approximates their true utility. 

In the noiseless ternary comparison model, a query presents two outcome profiles $\vx,\vy\in\Objset$ and receives one of three responses:
\[
\text{query response} =
\begin{cases}
    \vx\succ \vy, & \text{if } \inner{\vx-\vy}{\vw^*}>0,\\
    \vy\succ \vx, & \text{if } \inner{\vx-\vy}{\vw^*}<0,\\
    \vx\approx \vy, & \text{if } \inner{\vx-\vy}{\vw^*}=0.
\end{cases}
\]
Unless otherwise stated, our formal guarantees are stated with respect to a noiseless oracle. The exact oracle is useful for isolating the query complexity of linear preference elicitation. Human responses, however, may be noisy, especially when two profiles have nearly equal utility. In a noisy setting, we assume the user may answer with some margin of error, $\delta > 0$. Each response is assumed to be consistent with the true utility up to additive error $\delta$. We obtain the following implications from queries: $\vx \succ \vy \implies \inner{\vx - \vy}{\vw^*} \ge -\delta$ and $\vx \approx \vy \implies |\inner{\vx - \vy}{\vw^*}| \le \delta$. 

A simple packing argument shows that an $\epsilon$-approximation of the user's utility requires at least $\Omega( \obj \log(1/\epsilon))$ queries when $\obj\ge 2$ (\Cref{lem:linear-lower-bound} in the Appendix).

\input{text/cutting-plane}
\input{text/iterative_sieving}
\input{text/hybrid}

%, but we use the unweighted average as a natural baseline for estimating a community-level utility function. The aggregation requires fixing a common normalization, and one could perform an analogous aggregation using exclusively simplex normalized weights.

%% file: text/cutting-plane.tex
\subsection{A cutting plane method}

Cutting plane methods are a common approach to preference elicitation of linear utility functions. They maintain a space of weights that remain consistent with the observed pairwise responses. We leverage the \textit{analytic center cutting plane method (ACCPM)}~\citep{Atkinson95:Cutting}, a method that cuts the analytic center of the weight space at each iteration.  Like other cutting plane methods for preference elicitation, ACCPM does not have a bounded query complexity to reach a prescribed approximation error on the returned weight vector. However, the algorithm does provide an \textit{ex post} utility guarantee under the assumption that the user's answers are consistent up to a margin of $\delta$ with $\vw^*$. The \textit{ex post} certificate is measured by the diameter of the final weight space. We provide comprehensive details of ACCPM in~\Cref{sec:appendix-accpm}, including descriptions of the query synthesis and analysis of robustness to noisy responses.  

%% file: text/iterative_sieving.tex
\subsection{A bounded-query sieving algorithm}

The cutting plane method is effective in practice, but it does not converge in a provably bounded number of steps. When eliciting preferences from real users, we desire a guarantee that our algorithm terminates with the prescribed approximation accuracy using a number of queries that is not prohibitively large. We therefore develop for our setting an algorithm with a provable expected query complexity adapted from the recent \textit{iterative sieving} algorithm of~\citet{Cohen25:Combinatorial}. The algorithm operates on a finite cover of $\Objset$, and returns a point whose utility is approximately maximal under the user's linear utility.  In this subsection, we analyze the Euclidean-normalized representative of the user's preference direction. Throughout this subsection we write the true weight vector as
\[
    \vu^*\in\Objset
    \text{ and }
    \norm{\vu^*}_2=1.
\]
This is without loss since positive rescalings of a linear utility function induce the same ordering over outcomes. Unlike ACCPM, which maintains a weight space over weights, the sieving algorithm maintains a finite set of candidate outcome profiles. Because the maximizer over $\Objset$ is exactly $\vu^*$, eliminating suboptimal profiles also eliminates candidate preference directions.

The goal of the algorithm is to converge to a single profile in the finite set across iterations. The algorithm relies on two key concepts introduced by~\citet{Cohen25:Combinatorial}: the \emph{envelope} and the \emph{\CDfull}. Intuitively, at each iteration, after sorting a subset of profiles according to queries that the user has answered, the envelope identifies additional profiles that can be certified as suboptimal without further queries.  The \CDfull bounds how many profiles must be sorted at each iteration before many such inferences become possible. We provide the necessary background from the work of~\citet{Cohen25:Combinatorial} in~\Cref{sec:appendix-cohen}. We present the iterative sieving procedure in~\Cref{alg:linear-elicitation}.

We prove that the returned weight is an $\eps$-approximation of the Euclidean-normalized weight vector (\Cref{lem:sieving-epsilon-approximation}). Furthermore, the elicited weight vector approximates the utility for all profiles in $\Objset$ (\Cref{cor:sieving-utility-approximation}). Finally, the expected query complexity of the algorithm for our setting is $O(\obj^2 \log^2(1/\epsilon) \allowbreak \log (\obj\log (1/\epsilon))))$ (\Cref{lem:sieving-query-complexity}). All theoretical analyses pertaining to~\Cref{alg:linear-elicitation} are provided in~\Cref{sec:appendix-sieving}, along with further details on the algorithm and a new analysis of its noise robustness.

%% file: text/hybrid.tex
\subsection{A hybrid cutting plane and sieving algorithm}

The cutting plane method is effective in practice because it aggressively shrinks the weight space of plausible weights. However, it only provides an \textit{ex post} accuracy certificate through the diameter of the final weight space. When trying to obtain high-fidelity approximations, ACCPM can stall and fail to reduce the weight space's diameter. On the other hand, the sieving method provides a worst-case query bound for any level of fidelity, but it may require many comparisons when run from scratch. 

Therefore, in this subsection, we construct a hybrid approach that first runs the cutting plane algorithm for a fixed number of queries, and then uses the resulting cuts to warm start the sieving algorithm by further constraining the point removal before and during each iteration. Although the cutting plane phase maintains simplex-normalized weights and the sieving phase operates over Euclidean-normalized directions, this does not cause a conflict as positive rescaling preserves all pairwise comparisons.

The warm start uses the information learned by ACCPM to make the sieving phase more decisive. We first use the analytic center of the ACCPM weight space to choose an initial incumbent profile in $\cX$ by selecting the point that is closest directionally. Then, before and during sieving, we prune any profile that cannot beat the current incumbent under any weight vector still consistent with the accumulated cuts. In each sieving round, this test is strengthened by incorporating the ordering information from the sorted sample $\sigma$. The hybrid algorithm does not change the basic logic of sieving; it simply adds the information learned during the cutting plane phase to further remove points during each iteration. We present the hybrid approach as~\Cref{alg:hybrid-elicitation}.

\begin{algorithm}[!h]
\caption{Hybrid ACCPM and sieving preference elicitation}
\label{alg:hybrid-elicitation}
%\KwIn{target Euclidean accuracy $\eps$, cutting plane budget $T$}
%\KwOut{Euclidean-normalized preference-direction estimate $\widehat{\vu}\in\Objset$}
Run ACCPM for at most $T$ queries, obtaining $\cW_T$ and analytic center $\bar{\vw}_T$\\
\If{$\diam_2(\cW_T)\le \eps/\sqrt{\obj}$}{
    \Return{$\bar{\vw}_T/\norm{\bar{\vw}_T}_2$}
}
$\cX \coloneqq$ an $\epsilon/\sqrt{2}$-cover of $\Objset$ \\
$k \coloneqq c \obj \log (1/\eps)$ for a constant $c$  \\ 
$\vy \coloneqq \argmax_{\vx \in \cX} \inner{\vx}{\bar{\vw}_T}$ \\
Eliminate all points $\vx \in \cX \setminus \{\vy\}$ such that $\max_{\vw \in \cW_T} \inner{\vx-\vy}{\vw} \le 0$\\
\While{$|\cX|$ is more than $2k$}{
        Independently include each point of
        $\cX$ in a subset $\cY$ with probability $2k/|\cX|$ \\
        $\cY \gets \cY \cup \{ \vy\}$ \\
        $\sigma \gets$ result of sorting $\cY$ using $O(k \log k)$ expected comparisons
        \\
        $\vy \gets$ first (largest) element of $\sigma$ \\
        $\cW_{\sigma} \gets \{ \vw\in\cW_T \mid \inner{\vx_i - \vx_{i+1}}{\vw}\ge 0  \quad \forall i=1,\ldots,|\sigma|-1 \}$ \\
        Eliminate all points $\vx \in \cX \setminus \{\vy\}$ such that \\
        \Indp
        \Circled{1} $\dist(\vy - \vx, \envelope(\sigma)) < \eps^2/4$, or\\
        \Circled{2} $\max_{\vw \in \cW_\sigma} \inner{\vx-\vy}{\vw} \le 0$\\
        \Indm
}
\Return the maximum element, $\widehat{\vu}$, in $\cX$ found by a linear search
\end{algorithm}

Since we run ACCPM for a constant number of iterations, $T$, and the additional removal conditions before and during the sieving remove at least as many points as~\Cref{alg:linear-elicitation}, the hybrid approach inherits the bounded query complexity of~\Cref{alg:linear-elicitation}. Rule \Circled{1} is exactly the sieving elimination, and \Circled{2} only removes additional points without asking user queries. Therefore, the number of sieving rounds cannot be larger than in the ordinary analysis.

\begin{corollary}[Query complexity of~\Cref{alg:hybrid-elicitation}]
\label{cor:hybrid-query-complexity}
  The expected number of queries that \Cref{alg:hybrid-elicitation} asks is at most $$O(\obj^2 \log^2(1/\epsilon) \log (\obj\log (1/\epsilon))) + T.$$
\end{corollary}

What remains is to prove the correctness of the algorithm. 

\begin{restatable}[Weight approximation for~\Cref{alg:hybrid-elicitation}]{proposition}{HybridWeightApprox}
\label{lem:hybrid-epsilon-approximation}
Let $\vu^*\in\Objset$ be the Euclidean-normalized true weight vector, and let $\widehat{\vu}$ be the point returned by \Cref{alg:hybrid-elicitation}. Then
\[
    \norm{\widehat{\vu}-\vu^*}_2\le \eps.
\]
\end{restatable}

The pruning rule is safe because the true simplex-normalized weight remains feasible in $\cW_T$ and in every refined set $\cW_\sigma$. Thus, if $\max_{\vw\in\cW_\sigma}\inner{\vx-\vy}{\vw}\le0$, then in particular $\inner{\vx-\vy}{\vw^*}\le0$, so $\vx$ cannot improve on the incumbent under the true utility. And, as in~\Cref{cor:sieving-utility-approximation}, we also obtain the $\eps$-approximation of utility from the $\epsilon$-approximation of $\vu^*$. 

\begin{restatable}[Utility approximation for~\Cref{alg:hybrid-elicitation}]{corollary}{HybridUtilityApprox}
\label{cor:hybrid-utility-approximation}
For all $\vx\in\Objset$,
\[
    \left|
        \inner{\vx}{\widehat{\vu}}
        -
        \inner{\vx}{\vu^*}
    \right|
    \le \eps.
\]
\end{restatable}

\subsection{Simulation results for algorithm selection}

We evaluate the preference elicitation algorithms in simulation before conducting our user study. In addition to the algorithms presented in this paper, we evaluate a random baseline that samples query pairs uniformly at random and updates the weight-space estimate in the same way as ACCPM, the \textit{minimax regret} approach of~\citet{Boutilier04:Eliciting}, and the near-information-theoretic optimal algorithm of \citet{Ge24:Learning}, which we will refer to as \textit{attribute-wise binary search (AWBS)}. We explain the relationship between regret and utility approximation, and how to adapt minimax regret to obtain the desired approximation of the weight vector in~\Cref{sec:appendix-minimax-regret}. As we will show, our hybrid reaches each target approximation efficiently and demonstrates strong noise robustness. 

We use a comparison oracle with a randomly generated weight vector. For each value of $\eps$, we run each algorithm on the same 20 independently sampled random weight vectors, and operate in $\obj=6$ dimensions, mirroring the real elicitation for heart transplantation. To ensure a fair comparison, we report all approximation targets and errors in Euclidean distance between Euclidean-normalized preference directions. We present the results in~\Cref{fig:queries_vs_epsilon}. All algorithms perform substantially better than random querying, which does not converge within a budget of 250 queries. The iterative sieving (\Cref{alg:linear-elicitation}) requires over 100 queries as $\eps$ decreases. The ACCPM algorithm (\Cref{alg:accpm-elicitation}) is much more stable and uses up to around 50 queries. Notably, our hybrid algorithm (\Cref{alg:hybrid-elicitation}) uses fewer queries than ACCPM, iterative sieving, and minimax regret, and nearly the same number of queries as AWBS, demonstrating it is a practical algorithm that benefits from the combined techniques. 

\ifthenelse{\boolean{isSingleColumn}}{%
        \renewcommand{\imgwidth}{0.4\linewidth}
    }{%
        \renewcommand{\imgwidth}{0.6\linewidth}
    }%

\begin{figure}[!t]
    \centering
    \includegraphics[width=\imgwidth]{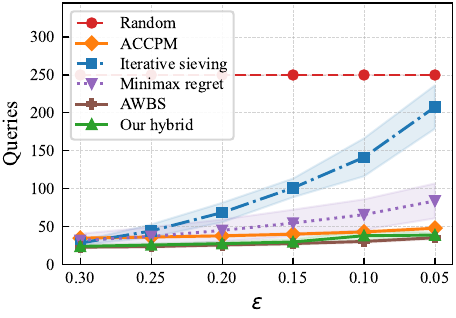}
    \caption{Queries as a function of $\epsilon$ using random weight vectors with a max budget of 250 queries. Shaded area shows standard deviation across 20 trials per $\epsilon$.}
    \label{fig:queries_vs_epsilon}
\end{figure}

\begin{figure}[b!]
    \centering
    \includegraphics[width=\imgwidth]{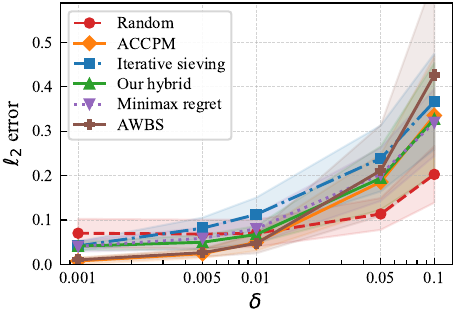}
    \caption{$\ell_2$ error when using a noisy oracle parameterized by $\delta$ with random weight vectors and a max budget of 250 queries. Shaded area shows standard deviation across 20 trials per $\delta$.}
    \label{fig:noisy_results}
\end{figure}

\begin{figure*}[!t]
    \centering
    \includegraphics[width=0.8\linewidth]{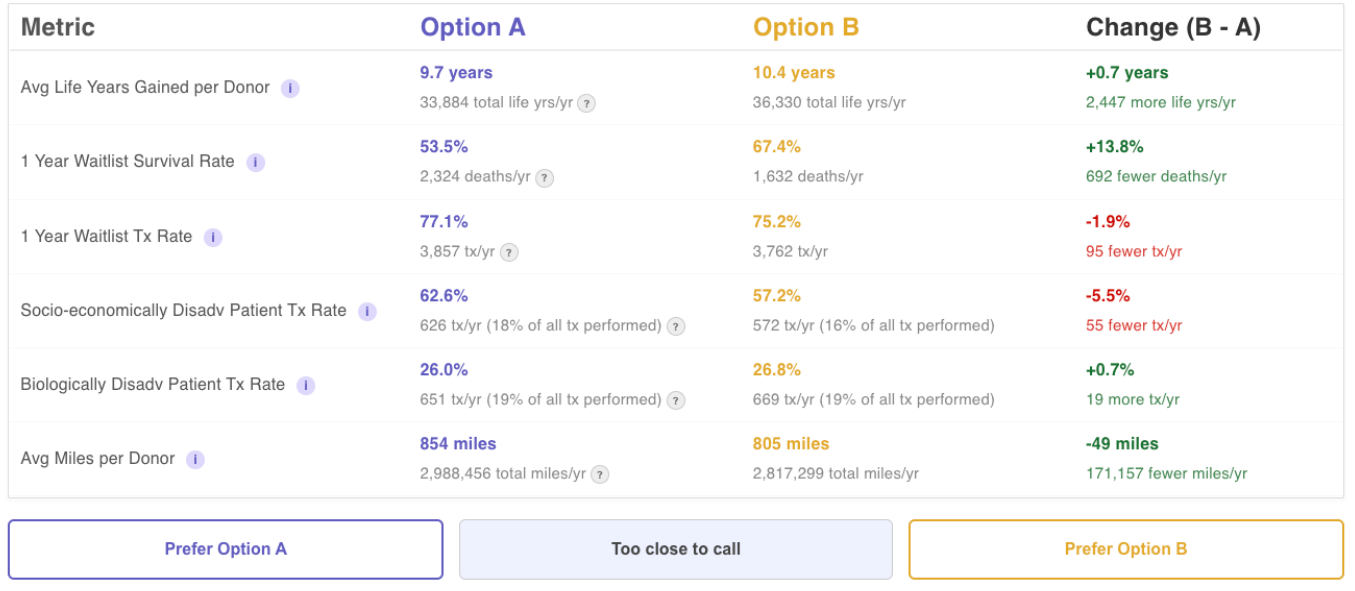}
    \caption{Elicitation interface used in our stakeholder study. Participants compare two outcome profiles and may indicate which profile they prefer or that the profiles are too close to call.}
    \label{fig:interface-query}
\end{figure*}

Next, we evaluate how our algorithms perform under noisy responses. We simulate a noisy oracle, parameterized by $\delta$, that answers a query between $\vx_i$ and $\vx_j$ by sampling a value $\zeta \in [-\delta,\delta]$ uniformly at random and reports $\vx_i\succ\vx_j$ if $\inner{\vx_i}{\vw^*}+\zeta \ge \inner{\vx_j}{\vw^*}$. We fix $\epsilon=0.05$ and $d=6$, and report the performance of each algorithm across different values of $\delta$ (\Cref{fig:noisy_results}). AWBS exhibits the largest degradation in Euclidean error as $\delta$ increases. Both ACCPM and our hybrid algorithm provide low-error approximations of the true weight vector for small values of $\delta$, and degrade gracefully as the noise increases to $0.1$, achieving $\ell_2$ error less than $0.35$ in the presence of large noise. Compared to AWBS, our hybrid algorithm has $30\%$ less error in the noisy setting. The noise robustness of the hybrid algorithm is a key property when deploying with human stakeholders as human responses may be noisy, and is one major motivating reason to use our algorithm over AWBS. Based on the experiments of this section, we will use our hybrid elicitation algorithm for the heart allocation experiments in the rest of this paper.

%% file: text/heart_transplant.tex
\section{Alignment of heart transplant allocation with human values}

We apply our preference elicitation algorithm to elicit stakeholder values to understand how humans trade off the different objectives of the heart transplant allocation system. We then use the elicited utility function for downstream policy optimization to generate an online allocation policy.

We consider $\obj=6$ objectives of the system. Each objective is an aggregate outcome of the allocation system, rather than a patient-level priority attribute. These objectives were identified and refined through interaction with clinical experts and motivated by prior work on heart transplant allocation~\citep{OPTN25:ContinuousDistribution,Papalexopoulos24:Reshaping,Anagnostides26:Position}. The six objectives are the following.
\begin{enumerate}
    \item \textit{Average life years gained per donor}: The average projected additional years of life a recipient will gain from a transplant compared to remaining on the waitlist.

    \item \textit{One year waitlist survival rate}: The percentage of waitlisted patients who do not receive a transplant that survive one year after being listed. 

    \item \textit{One year waitlist transplant rate}: The percentage of waitlisted patients who receive a transplant within one year of being listed. 

    \item \textit{Socio-economically disadvantaged patient transplant rate}: The percentage of waitlisted patients who are from a disadvantaged community that receive a transplant. A disadvantaged community is determined by a \textit{distressed community index (DCI)} $\ge 80$.

    \item \textit{Biologically disadvantaged patient transplant rate}: The percentage of waitlisted patients who are biologically disadvantaged that receive a transplant. A biologically disadvantaged patient is defined as a patient with \textit{calculated panel reactive antibody (cPRA)} $> 50\%$ or blood type O.

    \item \textit{Average miles per donor}: The average distance a donor heart travels to the recipient in miles.
\end{enumerate}

The objectives span measures of the clinical effectiveness, fairness, equity, and placement ease of the transplant system. We provide more details on these objectives in~\Cref{sec:appendix-objectives}. Each objective is first oriented so that larger values are preferred and then scaled to $[0,1]$. The elicitation algorithm operates on the corresponding normalized profiles, while the interface maps each coordinate back to its native scale for readability. Because the average miles per donor is the only objective that represents a cost to be minimized, its sign is reversed during the normalization phase so that all dimensions represent a maximization task. 

In addition to providing the pairwise responses, the users are optionally prompted to add a textual justification explaining their decisions. This feedback is used in combination with an LLM to ascertain whether a nonlinear, natural language-based preference function is superior to the linear model. We present and discuss these results in the next section.

\subsection{Eliciting and validating a utility function}

We leverage~\Cref{alg:hybrid-elicitation} to elicit the users' preferences. A snapshot of the elicitation interface is shown in~\Cref{fig:interface-query}. We run our algorithm with a tolerance of $\epsilon=0.05$. At this fidelity, the finite cover $\cX$ of $\Objset$ contains millions of candidate profiles. During the warm-start phase, we run ACCPM for $T=30$ iterations. For all point-elimination steps in~\Cref{alg:hybrid-elicitation}, we use linear programs (as detailed in \Cref{sec:appendix-sieving}), which we implement using the GNU Linear Programming Kit in C++. We run these filtering programs for millions of 6-dimensional points in seconds.  Hyperparameter settings and compute details are presented in~\Cref{sec:appendix-hyperparameters}. 

We include the results of our survey, which was shared with members of anonymized academic communities. Our study was deemed exempt by the \textit{Institutional Review Board(s) (IRB)} of the home institution(s). Informed consent was obtained from all participants prior to data collection. All data were anonymized and securely stored to protect the privacy and confidentiality of participants. We received 26 responses, which is comparable to the number of large transplant centers in the US (75+ transplants/year)~\citep{Colvin26:OPTN}. We report anonymized demographic information of respondents in~\Cref{tab:demographics}. With real users, the average number of queries to obtain a 0.05-approximation is roughly 40 (\Cref{tab:human-queries}), mirroring the simulated results. The algorithm converged in as few as 17 queries already during the ACCPM phase, and in the worst case took 74 queries. Users who did not finish the elicitation exited after 8.3 queries on average (\Cref{tab:human-dropout}); they were removed from the study.

\ifthenelse{\boolean{isSingleColumn}}{%
        \renewcommand{\tabwidth}{0.6\linewidth}
    }{%
        \renewcommand{\tabwidth}{\linewidth}
    }%

\begin{table}[h!]
    \centering
    \resizebox{\tabwidth}{!}{
    \begin{tabular}{ccccc}
        \toprule
        \textbf{Sample size} & \textbf{Mean queries} & \textbf{Std queries} & \textbf{Max queries}  & \textbf{Min queries} \\
        \midrule
        26 & 39.8 & 12.7 & 74 & 17 \\
        \bottomrule
    \end{tabular}}
    \caption{Distribution of queries asked with real users. Values indicate the queries asked by~\Cref{alg:hybrid-elicitation} and do not include the 10 validation queries.}
    \label{tab:human-queries}
\end{table}

\ifthenelse{\boolean{isSingleColumn}}{%
        \renewcommand{\tabwidth}{0.6\linewidth}
    }{%
        \renewcommand{\tabwidth}{0.8\linewidth}
    }%
\begin{table}[t!]
    \centering
    \small
    \resizebox{\tabwidth}{!}{
    \begin{tabular}{l|cc}
        \toprule
        \textbf{Tolerance ($\delta$)} & \textbf{Mean accuracy (no TCC)} & \textbf{Std accuracy (no TCC)} \\
        \midrule
        Exact & 75.8\% (80.0\%) & 16.3\% (16.4\%) \\
        0.01 & 77.5\% (81.7\%) & 15.1\% (15.1\%) \\        
        0.05 & 83.3\% (87.1\%) & 12.8\% (12.4\%) \\      
        0.1 & 91.2\% (93.6\%) & 9.3\% (8.8\%) \\        
        \bottomrule
    \end{tabular}}
    \caption{Accuracy of the learned linear utility model on validation queries with varying tolerances $\delta$. In parentheses we show the results on the validation queries where we exclude queries answered with \textit{"too close to call" (TCC)}.}
    \label{tab:human-validation}
\end{table}

After our elicitation algorithm terminated, we presented each user with an additional $10$ random queries to evaluate how well the algorithm had learned the utility function. Using the validation queries, we ascertain how accurate the linear utility model is by comparing the user's responses to the utility function that was learned for them. We report the accuracy as a function of the noise tolerance, $\delta$, which is the threshold for which the utility function must agree with the user (\textit{i.e.}, if the user answers $\vx \succ \vy$, then the linear model should identify $\inner{\vx - \vy}{\hvw} \ge -\delta$), in~\Cref{tab:human-validation}. A user's computed utility function correctly answers $75.8\%$ of validation queries on average when we allow no margin for error ($\delta = 0$) and interpret "too close to call" as exact equality. However, a participant's "too close to call" response may not correspond to exact equality under the linear utility. We therefore also report accuracy on the subset of validation responses expressing a strict preference (removing 12 queries). When we exclude these queries from the validation, the accuracy jumps to 80\%. With a margin of $0.1$, the linear utilities nearly perfectly agree with the user's decision making, with an average accuracy of $93.6\%$ when we exclude "too close to call" queries. 

To further test the efficacy of a linear utility model, we use the natural language explanations to examine whether LLMs, which are not restricted to a linear scoring rule, improve the accuracy on validation queries. We restrict this evaluation to participants who supplied natural language explanations and evaluate the linear baseline on the same subset ($N = 13$).

We evaluate three approaches that make use of LLMs. The \emph{basic LLM} receives a participant's elicitation responses and available explanations,
and predicts their validation responses. The \emph{heuristic LLM} additionally receives the participant's computed linear utility function to use as a heuristic. Motivated by the high accuracy of the linear model when the difference in utility of two points is large, the \emph{gated} method uses the linear prediction when the absolute predicted utility difference exceeds $0.1$, and otherwise uses the heuristic LLM prediction. We evaluate these approaches with three LLM backends. Prompts and implementation details are provided in \Cref{app:llm}.

\ifthenelse{\boolean{isSingleColumn}}{%
        \renewcommand{\tabwidth}{0.6\linewidth}
    }{%
        \renewcommand{\tabwidth}{1.0\linewidth}
    }%

\begin{table}[t!]
\centering
\resizebox{\tabwidth}{!}{
\begin{tabular}{llccc}
\toprule
Backend & Linear & Basic LLM & Heuristic LLM & Gated \\
\midrule
\texttt{gpt-5.4-mini}
 & 82.3\% & 64.6\% & 72.3\% &  74.6\% \\
\midrule
\texttt{mistral:7B}
  & 82.3\% & 70.8\% & 78.5\% & 77.7\% \\
\midrule
\texttt{gpt-6-luna}
 & 82.3\% &  73.1\% & 79.2\% &  77.7\% \\
\bottomrule
\end{tabular}}
\caption{Accuracy of the linear utility model and LLM methods on validation queries for users who supplied natural language explanations.}
\label{tab:llm-acc}
\end{table}   

\ifthenelse{\boolean{isSingleColumn}}{%
        \renewcommand{\tabwidth}{0.6\linewidth}
    }{%
        \renewcommand{\tabwidth}{1.0\linewidth}
    }%
    
\begin{table}[b!]
    \centering
    \resizebox{\tabwidth}{!}{
    \begin{tabular}{lcc}
        \toprule
        \textbf{Dimension} & \textbf{Mean weight} & \textbf{Std weight} \\
        \midrule
        Avg life years gained per donor & 0.37 & 0.17 \\
        One year waitlist survival rate & 0.28 & 0.20 \\
        One year waitlist tx rate & 0.12 & 0.09 \\
        Socio-economically disadv tx rate & 0.09 & 0.08 \\
        Biologically disadv tx rate & 0.10 & 0.16 \\     
        Avg miles per donor & 0.04 & 0.04 \\   
        \bottomrule
    \end{tabular}}
    \caption{Distribution of user weights.}
    \label{tab:human-weight-distribution}
\end{table}

\ifthenelse{\boolean{isSingleColumn}}
{\renewcommand{\imgwidth}{0.6\linewidth}}
{\renewcommand{\imgwidth}{1.0\linewidth}}
    
\begin{figure}[t!]
	\centering
    \includegraphics[width=\imgwidth]{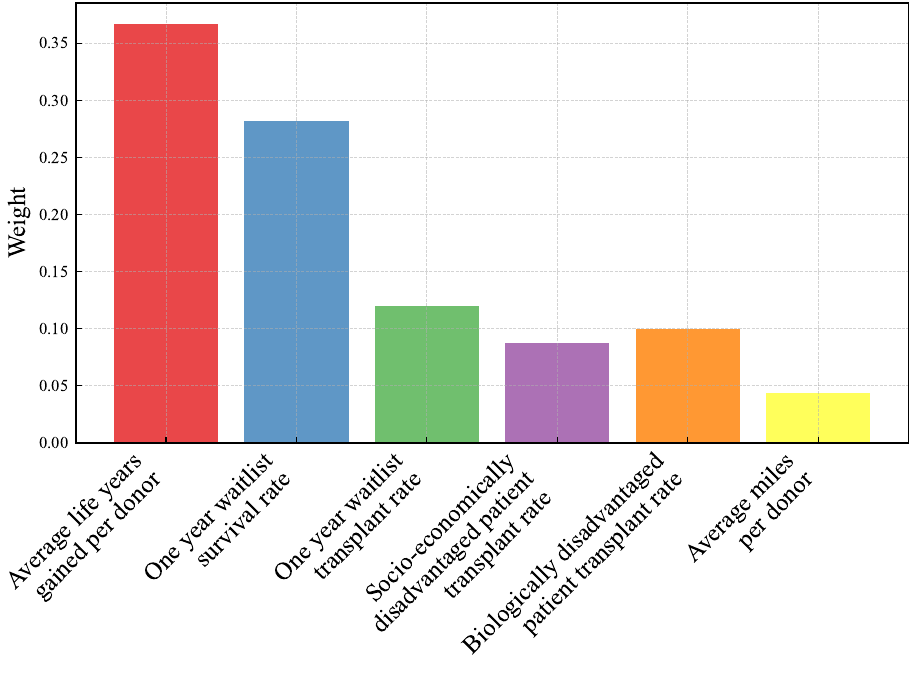}
	\caption{Average weight vector of user utility functions.}
	\label{fig:human-weights}
\end{figure}

We show in~\Cref{tab:llm-acc} the results of these LLM methods on the validation queries. On this subset, the linear model achieves $82.3\%$ accuracy, compared with at most $79.2\%$ for the evaluated LLM configurations. Providing the computed linear utility as a heuristic improves accuracy for all three backends relative to the basic LLM. However, neither direct LLM prediction nor the gated method improves on the linear model. We also observe that the accuracy of the linear model on the subset of users who supplied natural language explanations is higher compared to the entire validation set of users. This may suggest that users who supplied natural language explanations were more engaged with the survey and consistent in responding to queries based on their preferences.

These results support the use of linear utilities as a predictive model in our setting. Although the LLM-based predictors have access to participants' explanations, none improves observed validation accuracy over the elicited linear utility. 

Moving forward with the linear model, we report the simplex-normalized aggregated weights from community members in~\Cref{tab:human-weight-distribution} and~\Cref{fig:human-weights}. We use the mean aggregation method, but this is just one of many possible aggregation strategies. We provide more details on aggregation methods in~\Cref{sec:appendix-weight-aggregation}. Mean aggregation is very natural due to preservation of $\epsilon$-accuracy of the aggregated utility function, as well as resulting in a welfare-maximizing strategy.

The average life years gained per donor\footnote{These results suggest that maximizing the life years of organ placement is a priority, and provide motivation for prior computational methods that do that~\citep{Berrevoets20:OrganITE,Berrevoets21:Learning,Anagnostides25:Policy,Zilberstein26:Near}.} was the most  prioritized objective with a weight of 0.37 followed by the one year waitlist survival rate. The transplant rates for disadvantaged patients and transport efficiency were the least prioritized with weights at or below 0.1. 

% make footnote; These results suggest that maximizing the life years of organ placement is a priority, and provide strong motivation for prior computational methods that do such~\citep{Berrevoets20:OrganITE,Berrevoets21:Learning,Anagnostides25:Policy,Zilberstein26:Near}.

\input{text/policy_optimization}

%% file: text/policy_optimization.tex
\subsection{Closing the loop: Human-value-aligned policy optimization}

Having elicited and aggregated a utility function that captures how humans trade off the objectives of the heart transplantation system, we show that we can automatically optimize an online heart transplant allocation policy that is near-optimal for the human preferences. We intentionally treat policy optimization as a downstream consumer of the learned utility function; we stress that any optimization framework could be substituted here. We adapt the self-supervised imitation learning method of~\citet{Zilberstein26:Learning} for learning a data-driven online matching policy. This method computes, over historical training periods, the optimal hindsight allocation via integer programming given a utility function. An online policy is then learned by training a neural network to imitate the decisions of the omniscient allocation. 

We compare our optimized policy to the current US \textit{status quo} policy and the standard greedy benchmark which selects the patient to transplant with the highest predicted life years gained. We leverage a simulator that uses real, historical data from UNOS dating back to 1987. We optimize our policy using data from January to March of 2019 for training, and evaluate on the unseen, real trajectory of events for the months April to December in 2019. This period contains thousands of transplant allocations. We report the average objective and utility values per month in~\Cref{tab:policy_optimization}.  We provide more details on the simulation environment and policies in~\Cref{sec:appendix-policy-optimization}.

\ifthenelse{\boolean{isSingleColumn}}
{\renewcommand{\tabwidth}{0.6\linewidth}}
{\renewcommand{\tabwidth}{1.0\linewidth}}
    
\begin{table}[t!]
    \centering
    \resizebox{\tabwidth}{!}{
    \begin{tabular}{l|ccc|r}
    \toprule
    \textbf{Metric} & \textbf{Status quo} & \textbf{Greedy} & \textbf{Ours} & \textbf{Hindsight OPT} \\
    \midrule
    Avg life years gained per donor & 4.53 & 10.21 & 10.10 & 10.06 \\
    One year waitlist survival rate & 87.3\% & 87.4\% & 87.4\% & 87.6\% \\
    One year waitlist tx rate & 45.4\% & 47.8\% & 47.5\%  & 50.3\% \\
    Socio-economically disadv tx rate & 58.8\% & 60.8\% & 66.2\% & 67.3\% \\
    Biologically disadv tx rate & 58.5\% & 56.2\% & 59.5\% & 60.0\% \\
    Avg miles per donor & 610 & 745 & 731 & 322 \\
    \midrule
    Elicited utility function value & 0.308 & 0.430 & \textbf{0.441} & 0.466 \\
    Competitive ratio to OPT & 0.66 & 0.92 & \textbf{0.95} & 1.00 \\
    \bottomrule
    \end{tabular}
    }
    \caption{Average results for policies in simulation from April to December 2019.}
    \label{tab:policy_optimization}
\end{table}

Our optimized policy successfully trades off the six objectives compared to the US \textit{status quo} and the greedy allocation policy. Our policy is also near-optimal, achieving a competitive ratio of $95\%$. In other words, our optimized online policy performs nearly as well as an optimal algorithm that has perfect foresight of future events. This competitive ratio is higher than those of both the \textit{status quo} (66\%) and the greedy allocation (92\%). The greedy allocation is a strong baseline as the life years gained per donor is the highest weighted objective by humans. However, our policy successfully trades off this objective to achieve higher values in other objectives that humans care about; it improves fairness via transplant rates for socio-economically and biologically disadvantaged patients, ultimately improving the overall utility.

%% file: text/conclusions.tex
\section{Conclusions}
\label{sec:conclusions}

We studied the problem of preference elicitation for policy optimization. We developed a novel two-phase preference elicitation algorithm that combines the empirical efficiency of cutting plane methods with the theoretical guarantees of iterative sieving, which substantially reduced the number of queries required in practice. We then applied our framework to heart transplant allocation, with the conceptual distinction to elicit preferences over allocation outcomes rather than individual allocation decisions. Through a user study, we demonstrated that accurate utility functions can be learned with approximately 40 pairwise comparison queries. Using natural language explanations supplied by users, we validated the linear representation against LLM-based methods, none of which improved the accuracy on held-out validation queries compared to the linear model. We then leveraged the elicited utility function for downstream policy optimization. We learned an online allocation policy that was substantially more aligned with human values, achieving $95\%$ of the hindsight optimal utility, compared to the current US \textit{status quo} policy's $66\%$. By separating the elicitation of stakeholder values (\textit{ends}) from the optimization of allocation policies (\textit{means}), our framework provides a practical, end-to-end path toward aligning AI decision-making systems with human values.

%% file: text/acks.tex
\section*{Acknowledgments}

Tuomas Sandholm and his PhD students Ioannis Anagnostides and Itai Zilberstein are supported by NIH award A240108S001, the Vannevar Bush Faculty Fellowship ONR N00014-23-1-2876, and National Science Foundation grant RI-2312342. Itai Zilberstein is also supported by the NSF Graduate Research Fellowship Program under grants DGE2140739 \& DGE2631988. Arman Kilic is supported by NIH RO1 grant 5R01HL162882-03 which contributed to the funding for completion of this project. Arman Kilic is a speaker and consultant for Abiomed, Abbott, 3ive, and LivaNova, and founder and owner of QImetrix. All additional authors have no financial relationships to disclose. Any opinions, findings, and conclusions or recommendations expressed in this material are those of the author(s) and do not necessarily reflect the views of the funding agencies. 

%% file: text/appendix.tex
\setcounter{table}{0}
\renewcommand{\thetable}{A\arabic{table}}

\setcounter{figure}{0}
\renewcommand{\thefigure}{A\arabic{figure}}

\numberwithin{theorem}{section}
\numberwithin{lemma}{section}
\numberwithin{definition}{section}
\numberwithin{corollary}{section}
\numberwithin{proposition}{section}

\input{text/related}
\input{text/appendix/appendix-cohen-addad}

\input{text/appendix/appendix-accpm}

\input{text/appendix/appendix-sieving}
\input{text/appendix/minimax-regret}
\input{text/appendix/appendix-aggregation}
\input{text/appendix/appendix-proofs}
\input{text/appendix/appendix-heart-transplant-elicitation}
\input{text/appendix/appendix-figs-tables}

\input{text/appendix/appendix-hyperparameters}
\input{text/appendix/appendix-policy-optimization}
\input{text/appendix/appendix-llms}
\input{text/appendix/irb}

%% file: text/related.tex
\section{Further related work}
\label{sec:related}

A key reference point for our work is the continuous distribution framework for organ allocation~\citep{Papalexopoulos24:Reshaping}. In continuous distribution, candidates are prioritized using a \textit{composite allocation score (CAS)}, which aggregates allocation attributes such as medical urgency, expected post-transplant benefit, placement efficiency, proximity, and pediatric priority into a linear weighted scoring rule~\citep{Cummiskey25:Understanding}.  Continuous distribution used preference elicitation to learn weights for CAS using the \emph{analytic hierarchy process (AHP)}~\citep{Saaty77:Scaling}. We include an excerpt from the official release that details AHP for continuous distribution~\citep{OPTN_Lung_CD}. 

\begin{quote}
    ``Anyone can participate in the AHP exercises for continuous distribution, which will occur for all organ types. In an AHP exercise for continuous distribution, participants are shown pairs of attributes (blood type vs. distance, for example) that will be used to prioritize candidates. The AHP participant must decide, if all else is considered equal, which of the two attributes is more important than the other when prioritizing a candidate for an organ.''
\end{quote}

For kidney exchange, recent preference elicitation approaches have focused on eliciting preferences over policies rather than decisions~\citep{Vayanos26:Robust}. This work also considers linear scoring rules and restricts the policy class to simulator-based outcomes, performing the elicitation with a current feasible set of policies rather than a frontier-agnostic domain.

Minimax regret has been used to choose robust decisions and to guide further elicitation when the utility function is only partially known~\citep{Boutilier06:Constraint,Boutilier13:Computational}. Similar ideas have been applied to auctions, where elicitation is used to learn preferences relevant to winner determination and mechanism design~\citep{Conen01:Preference,Boutilier04:Eliciting,Sandholm06:Preference}. These approaches are close in spirit to our methods as they reason about decision quality under a set of utility functions consistent with the elicited information. However, regret-based elicitation typically focuses on selecting a good alternative from a fixed set. We leverage the \textit{analytic center cutting plane method (ACCPM)}, which iteratively adds separating hyperplanes in the weight space through the analytic center of the region~\citep{Atkinson95:Cutting}. Cutting plane methods have also been used in adaptive conjoint analysis, where each query imposes a new linear constraint that shrinks the weight space~\citep{Toubia03:Fast,Toubia04:Polyhedral}. 

An alternative method is Bayesian preference elicitation, which maintains a posterior distribution over user preferences and selects queries that are informative under that posterior~\citep{Guo10:Real,Eric07:Active,Gonzalez17:Preferential,Lin22:Preference}. The empirical performance of Bayesian elicitation methods depends on modeling assumptions such as the prior, likelihood, acquisition function, and posterior approximation. In addition, Bayesian methods can become computationally demanding, and they generally provide posterior uncertainty estimates rather than deterministic worst-case guarantees. More recently, model-free techniques were proposed~\citep{Martin23:Model}. Unlike much of the Bayesian and regret-based preference elicitation literature, which typically focuses on selecting a good alternative from a fixed decision set, our goal is to recover a reusable preference function that can be applied to downstream optimization over a changing frontier of heart allocation policies.

%\paragraph{LLM-driven preference elicitation}
There is emerging research on using \textit{large language models (LLMs)} as proxies for human preferences, and we build on some of this recent work. LLMs successfully served as simulated economic agents or human-subject proxies in some experimental settings~\citep{Horton23:Large}. \citet{Huang25:Accelerated} study accelerated preference elicitation with LLM-based proxies, using language models to reduce the number of human queries required. In allocation and matching settings, \citet{Soumalias24:Machine} use pairwise comparisons for course allocation, illustrating the usefulness of comparison-based elicitation in complex allocation environments.

%Our LLM-driven method follows this general direction. While our human preference elicitation learns a linear utility function, we leverage LLMs to account for non-linear utility functions. After observing a small number of pairwise comparisons and textual explanations from the user, the model acts as a proxy oracle for additional preference queries or as a direct value oracle over outcome profiles. The LLM is not intended to replace human value judgments entirely. Instead, it serves to amortize the cost of elicitation by extrapolating from a limited set of human-provided comparisons and rationales. 

%% file: text/appendix/appendix-cohen-addad.tex
\section{Additional background from \citet{Cohen25:Combinatorial} }
\label{sec:appendix-cohen}

In this section, we provide the required definitions and results from the work of \citet{Cohen25:Combinatorial}. Their results for linear optimization with comparison oracles are concerned with finding the \textit{minimum} element in a set, whereas we seek to find the \textit{maximum}. This discrepancy does not change any of the results, and all their definitions and theorems can be applied to our setting. 
 
\subsection{Envelope and conic dimension}
We start with the definitions of the \emph{envelope} and \emph{\CDfull} of a point set. The envelope leverages an ordering of the points $\vx_1, \ldots, \vx_k$ induced by $\vw^*$, denoted $\sigma$. The envelope captures all implications of the ordering that can be derived without further queries.

\begin{definition}
\label{def:envelope}
For a sequence of points $\sigma = (\vx_1, \ldots, \vx_k) \in (\R^d)^k$:
\begin{enumerate}
    \item The \emph{envelope} of $\sigma$ is
    \[
        \envelope(\sigma) := \cone(\{\vx_j - \vx_i\}_{1 \le i < j \le k})
    \]
    \[
        = \cone(\{\vx_{i+1} - \vx_i\}_{1 \le i < k}).
    \]
    \item The points are \emph{conically independent} if, for each $t \in \{2, \ldots, k\}$,
    \[
        \vx_t - \vx_1 \notin \envelope(\sigma^{1:t-1}),
    \]
    where $\sigma^{1:t}$ denotes the prefix consisting of the first $t$ elements of $\sigma$.
\end{enumerate}
\end{definition}

\begin{definition}
  \label{defn:conic-dimension}
  The \emph{\CDfull} of a set $\cY \subseteq \R^d$, denoted
  $\CDmath(\cY)$, is the largest integer $k$ for which there exists a
  length-$k$ conically independent sequence
  $\sigma = (\vx_1, \dots, \vx_k) \in \cY^k$.
\end{definition}

Intuitively, the \CDfull is the maximum sequence length, $t$, in which one cannot infer the ordering of points $\vx_i$ for $i < t$.

\subsection{Basic subsequences}
\label{sec:basic-subsequences}

The notion of basic subsequence is necessary for the proof of the query complexity of~\Cref{alg:linear-elicitation}. Given a set $\cY\subseteq \R^d$ with $\CDmath(\cY)=k$ and $\sigma = (\vx_1, \dots, \vx_n) \in \cY^n$, we define the process that maintains a subsequence $\pi_t$ of $\sigma$ for $t \in [0,n]$. We define $\pi_0 = \ip{}$, and at step $t+1$, the process updates $\pi_{t+1}$ using
\[
\pi_{t+1} = 
\begin{cases}
	\pi_t \circ \vx_{t+1}, & \text{if } \vx_{t+1} - \vx_1 \notin \envelope(\pi_t) \\
	\pi_t, & \text{otherwise}.
\end{cases}
\]
See that $\pi_{t+1}$ always contains $\pi_t$ as a prefix and $\pi_{t+1}$ contains at most one more element, $|\pi_n| \leq k$. The second claim follows from the definition of \CDfull. We define the ``basis'' of $\sigma$ as $B(\sigma) := \pi_n$, which is the subsequence of $\sigma$ obtained by this process.

\begin{lemma}
	\label{lem:basis-cone-property}
	Given a sequence $\sigma = (\vx_1, \dots, \vx_n) \in \cY^n$, we have:
	\begin{align}
		\cone(\{\vx_j - \vx_1\}_{1 \leq j \leq n}) \subseteq \envelope(B(\sigma)) \subseteq \envelope(\sigma).
	\end{align}
\end{lemma}

\subsection{Iterative sieving algorithm}

Next, we provide the algorithm for finding a linear optimizer using a small number of queries. 

\begin{algorithm}[h!]
  \caption{Iterative Sieving Algorithm}
  \label{alg:sample-remove-recurse}
  \While{$|\cP|$ is more than $O(k)$}{
    \label{item:alg-step-1} Independently include each point of
    $\cP$ in a subset $\cY$ with probability $2k/N$ \\
    $\sigma \gets$ result of sorting $\cY$ using $O(k \log k)$ expected comparisons
    \\
    $\vy \gets$ first (smallest) element of $\sigma$ \\
    \label{item:alg-step-2} Eliminate all points
    $\vx \in \cP \setminus \{\vy\}$ such that
    $\vx - \vy \in \envelope(\sigma)$
  }
  \label{item:alg-step-3} Sort $\cP$ using a brute-force algorithm.
\end{algorithm}

To prove the query complexity of the iterative sieving algorithm, it suffices to show that the number of points in $\cP$ is halved in expectation during each iteration. We include their proof of this claim.

\begin{restatable}[Sieving Lemma]{lemma}{Sieve}
  \label{lem:alg-progress}
  The expected number of points that are eliminated from $\cP$ at
  step~\ref{item:alg-step-2} is at least $|\cP|/2$.
\end{restatable}

\begin{proof}[Proof of~\Cref{lem:alg-progress}]
	For the analysis, consider the sequence $(\vx_1, \dots, \vx_N)$ such
	that $\langle \vw^*, \vx_i \rangle \leq \langle \vw^*, \vx_{i+1} \rangle$
	for all $1\leq i\leq N-1$. Note that $\sigma$ is distributed as a
	random subsequence of $(\vx_1, \dots, \vx_N)$, where each element is
	selected independently with probability $2k/N$.
	
	Consider a slightly modified, more lenient version of the algorithm
	that eliminates a point $\vx_t \in \cP \setminus \{\vy\}$ in
	Step~\ref{item:alg-step-2} only if
	$\vx_t - \vy \in \envelope(B(\sigma_{t-1}))$ where $\sigma_t$ is the
	prefix of $\sigma$ restricted to elements $\{\vx_i\}_{1\leq i\leq t}$;
	recall the definition of the basic subsequence $B(\cdot)$ from
	\Cref{sec:basic-subsequences}. This modification ensures that
	elimination of $\vx_t$ depends only on the coin tosses at indices in
	$[t-1]$ and the sequence $(\vx_1,\dots,\vx_t)$. This modified algorithm
	eliminates no more points than the original algorithm because
	$\envelope(B(\sigma_{t-1}))\sse \envelope(\sigma_{t-1})\sse
	\envelope(\sigma)$---each inequality just uses that the cone of a
	smaller set of vertices is smaller. Hence, it suffices to prove that
	the expected number of points eliminated by the modified algorithm
	is at least $N/2$.
	
	We now use the principle of deferred randomness to analyze this
	version of the algorithm. Let $E_t$ be the ``evolving set'' of elements in $\{\vx_i\}_{1\leq i\leq t}$ that are not eliminated by the modified algorithm after it has seen the elements $\vx_1,\vx_2\dots,\vx_t$ in that order. Remember that $\sigma_t$ is the subsequence of $(\vx_1,\dots,\vx_t)$ corresponding to elements that are sampled. We can simulate the process of generating the sets $E_t$ and $\sigma_t$ in the following way:
	\begin{enumerate}
		\item Start with $E_0, \sigma_0=\ip{}$.
		\item At step $1\leq t\leq N$, toss a biased coin with success probability $2k/N$. Update 
		\begin{align*}
			\sigma_{t} &= 
			\begin{cases}
				\sigma_{t-1} \circ \vx_{t}, & \text{if coin flip succeeds at } $t$ \\
				\sigma_{t-1}, & \text{otherwise}.
			\end{cases} 
			\intertext{and}
			E_{t} &= 
			\begin{cases}
				E_{t-1} \circ \vx_{t}, & \text{if } \vx_{t} - \vy \notin \envelope(B(\sigma_{t-1})) \\
				E_{t-1}, & \text{otherwise}.
			\end{cases}
		\end{align*}
		where $\vy$ is the first element sampled.
	\end{enumerate}
	We observe that 
	\begin{align*}
		|B(\sigma_{t})|-|B(\sigma_{t-1})| &=\mathbf{1}[\vx_t-\vy \notin \envelope(B(\sigma_{t-1})) \\
		&\land \textnormal{coin flip succeeds at }t] \\
		&=(|E_t|-|E_{t-1}|)\cdot\mathbf{1}[\textnormal{coin flip succeeds at }t].\\
		\intertext{
			Taking expectations on both sides gives 
		}
		\E[|E_t|-|E_{t-1}|]&=\frac{N}{2k}\cdot \E[|B(\sigma_{t})|-|B(\sigma_{t-1})|].
		\intertext{Summing this over $1\leq t\leq N$ gives}
		\E[|E_{N}|]&=\frac{N}{2k}\cdot \E[|B(\sigma)|]\leq
		\frac{N}{2k}\cdot k =\frac{N}{2}.  
	\end{align*}
\end{proof}

\subsection{$\eps$-Approximation for bounded sets}
% \subsubsection{$\eps$-Approximate Solutions for Bounded Sets}
% \label{sec:bounded-sets}
The \CDfull is hard to bound in the exact sense for many point sets. However, it is possible to bound an approximation of the \CDfull for the relevant sets $\cP \subset \BB_d = \{x \in \R^d; \norm{x}_2\leq 1\}$. Using the approximate \CDfull produces an algorithm that returns an $\eps$-approximation of the minimizer. Given a set $S \subset \R^d$ and a point $\vx \in \R^d$, define $\dist(\vx,S) = \min_{\vy \in S} \norm{\vx-\vy}_2$.

\begin{definition}

	\label{defn:approx-conic-dimension}
	The \emph{$\eps$-approximate \CDfull} of a \emph{set}
	$\cY \subseteq \mathbb{R}^d$ denoted by $\CDmath_\eps(\cY)$ is the largest integer $k$ for which
	there exists a length-$k$ sequence
	$\sigma = (\vx_1, \dots, \vx_{k}) \in (\cY)^k$ of points from $\cY$, such
	that for each $t \in \{2, \ldots, k\}$ we have
	\begin{align}
		\dist(\vx_t - \vx_1, \envelope(\sigma^{1:t-1})) \geq \eps.
	\end{align}
\end{definition}

The previous concepts naturally extend to the approximation of the \CDshort. 
For example, the basic subsequence $B_\eps(\sigma)$ is modified to append an element if its distance to the set of elements is at least $\eps$. 

Let $k = \CDmath_\eps(\cY)$. Then a variation of  Algorithm \ref{alg:sample-remove-recurse} can find an approximate minimizer. This variation has two subtle changes. The first always includes the minimum element of $\sigma$ from the prior iteration in $\cY$. The second removes all points such that $\dist(\vx-\vy, \envelope(\sigma)) < \eps$ instead of removing based on pure membership. 

\begin{corollary}
\label{cor:epsilon-query}
	If $\norm{\vw^*} \leq 1$ and $k = \CDmath_\eps(\cY)$, the variant of Algorithm \ref{alg:sample-remove-recurse} described above returns a point $\hat \vx$ such that $\langle \vw^*, \hat \vx \rangle \leq \langle \vw^*,  \vx \rangle + \eps$ for all $\vx \in \cP$ using $O(k \log k \log |\cP|)$ comparisons.
\end{corollary}

The last result we require is a bound on the approximate \CDshort of the relevant point sets. 

\begin{lemma}
	\label{lem:approx-conic-dim-upperbound}
	The $\eps$-approximate \CDfull of any set of points $\cP \subset \BB_d$ is the smallest $k$ such that
	$2^k (\eps/2)^d > (2k+1)^d$.
\end{lemma}
We refer the reader to Section A.3 of \citet{Cohen25:Combinatorial} for a proof of this claim.

%% file: text/appendix/appendix-accpm.tex
\section{Further details on ACCPM}
\label{sec:appendix-accpm}

The heuristic elicitation method maintains a weight space of weights that remain consistent with the observed pairwise responses. We use the same noise parameter $\delta$ from the $\delta$-tolerant response model above. If all observed responses are $\delta$-consistent with $\vw^*$, then $\vw^*$ remains in the weight space throughout the algorithm.

Let $\cS^+$ denote the set of oriented cut normals induced by strict preference responses after $t$ queries. If the user reports $\vx_i\succ\vx_j$, we add $\vx_i-\vx_j$ to $\cS^+$; if the user reports $\vx_j\succ\vx_i$, we add $\vx_j-\vx_i$ to $\cS^+$. Likewise, let $\cS^0$ denote the set of normals $\vx_i-\vx_j$ induced by indifference responses. Given a tolerance $\delta \ge 0$, the weight space at iteration $t$ is $\cW_t = \{ \vw\in\R^\obj_{\ge 0} \mid \norm{\vw}_1=1, \inner{\va}{\vw}\ge -\delta \;\forall \va\in\cS^+, \left|\inner{\va}{\vw}\right|\le \delta \;\forall \va\in\cS^0 \}$.
Because the algorithm may query any pair of profiles in $\Objset$, it can synthesize a rich family of cuts of the weight space. We leverage the \textit{analytic center cutting plane method (ACCPM)}, a heuristic method that cuts the analytic center of the weight space at each iteration. For a weight space with nonempty relative interior, define the logarithmic barrier
\begin{align*}
    \Phi_t(\vw)
    ={}&
    -\sum_{i=1}^{\obj}\log w_i
    -\sum_{\va\in\cS^+}\log\left(\inner{\va}{\vw}+\delta\right) \\
    &
    -\sum_{\va\in\cS^0}\log\left(\delta-\inner{\va}{\vw}\right)
    -\sum_{\va\in\cS^0}\log\left(\delta+\inner{\va}{\vw}\right).
\end{align*}
The barrier above is written for the case $\delta>0$. In the noiseless setting in which $\delta=0$, responses $\vx_i\approx\vx_j$ are treated as affine constraints $\inner{\va}{\vw}=0$. The analytic center is any solution of
\begin{equation}
\label{eq:analytic-center}
    \bar{\vw}_t
    \in
    \argmin_{\vw\in \operatorname{relint}(\cW_t)} \Phi_t(\vw).
\end{equation}
When $\cS^+=\cS^0=\emptyset$, this reduces to the analytic center of the simplex, namely the uniform weight vector. We can solve \Cref{eq:analytic-center} efficiently using Newton's method online.

Each query induces a cut normal of the form $\vx_i-\vx_j$ for some $\vx_i,\vx_j\in\Objset$. Our algorithm chooses queries by  perturbing the current analytic center. Let $\vb_t=\frac{\bar{\vw}_t}{\|\bar{\vw}_t\|_2}\in\Objset$
be the Euclidean-normalized analytic center. We choose a unit direction $\vd_t$ satisfying $\inner{\vd_t}{\bar{\vw}_t}=0$ and query two profiles of the form
\begin{equation}\label{eq:accpm-query}
    \vx^+
    =
    \frac{\vb_t+\tau\vd_t}
         {\|\vb_t+\tau\vd_t\|_2}
    \text{ and }
    \vx^-
    =
    \frac{\vb_t-\tau\vd_t}
         {\|\vb_t-\tau\vd_t\|_2}.
\end{equation}
Because $\vd_t$ is orthogonal to $\bar{\vw}_t$, the induced cut is balanced around the analytic center. We choose $\vd_t$ as the direction of smallest local curvature of the analytic-center barrier. The step size $\tau$ is chosen to target a prescribed query separation, $\gamma > 0$, while ensuring that both profiles remain in the nonnegative orthant. We provide further details on the query synthesis at the end of this section. The iterative ACCPM procedure is outlined in~\Cref{alg:accpm-elicitation}.

\begin{algorithm}[t]
\caption{ACCPM preference elicitation}
\label{alg:accpm-elicitation}
Initialize $\cS^+\gets\emptyset$, $\cS^0\gets\emptyset$, and $\cW_0\gets\cW$\\
\For{$t=1,\dots,T$}{
    Compute the analytic center $\bar{\vw}_{t-1}$ of $\cW_{t-1}$ from~\eqref{eq:analytic-center} \\
    \If{$\diam_2(\cW_{t-1})\le \eps$}{
        \Return{$\bar{\vw}_{t-1}$}\;
    }
    Choose a query $(\vx_i,\vx_j)$ according to~\eqref{eq:accpm-query}\\
    Query user preference of $\vx_i$ and $\vx_j$ \\
    Update $\cS^+$ or $\cS^0$ based on the user's response and compute $\cW_t$\\
}
\Return any \(\widehat{\vw}_T\in \cW_T\)
\end{algorithm}

The ACCPM algorithm does not have a bounded query complexity to reach a prescribed weight space diameter. However, the algorithm does provide an \textit{ex post} utility guarantee under the assumption that the user's answers are consistent up to a margin of $\delta$ with $\vw^*$. The \textit{ex post} certificate is measured by the diameter of the final weight space. 

\begin{restatable}[\textit{Ex post} weight approximation for~\Cref{alg:accpm-elicitation}]{proposition}{AccpmWeightApprox}
\label{lem:accpm-weight-approximation}
If Algorithm~\ref{alg:accpm-elicitation} terminates with $\diam_2(\cW_t)\le\eps$, then $\norm{\vw^*-\widehat{\vw}}_2\le\eps$.
\end{restatable}

Moreover, we obtain an \textit{ex post} guarantee on the utility function. 

\begin{restatable}[\textit{Ex post} utility approximation for~\Cref{alg:accpm-elicitation}]{corollary}{AccpmUtilityApprox}
\label{cor:accpm-utility-approximation}
If Algorithm~\ref{alg:accpm-elicitation} terminates with $\diam_2(\cW_t)\le\eps$, then 
$$\sup_{\vx \in \Objset} |\inner{\vx}{\vw^*} - \inner{\vx}{\hvw}| \le \epsilon.$$
\end{restatable}

We implement the stopping condition by maintaining a bounding box over the weight space, which gives an easily computable upper bound on $\diam_2(\cW_t)$. To obtain a Euclidean-normalized preference direction, terminating with $\diam_2(\cW_t)\le \frac{\eps}{\sqrt{\obj}}$ would give an $\eps$-approximation of the Euclidean-normalized weight. 

\paragraph{Query synthesis}
Although the continuous formulation permits arbitrary queries in $\Objset$, in practice we avoid extreme profiles such as coordinate vectors. Instead, each query is a local perturbation of the
current analytic center. Let
\[
    \vb_t
    =
    \frac{\bar{\vw}_t}{\norm{\bar{\vw}_t}_2}
    \in\Objset
\]
denote the Euclidean-normalized analytic center at iteration $t$. After computing $\bar{\vw}_t$, we choose a trade-off direction $\vd\in\R^\obj$ satisfying
\[
    \inner{\vd}{\bar{\vw}_t}=0 \text{ and }
    \norm{\vd}_2=1.
\]
The query presented to the user is
\[
    \vx^+
    =
    \frac{\vb_t+\tau\vd}
         {\norm{\vb_t+\tau\vd}_2}
    \text{ and }
    \vx^-
    =
    \frac{\vb_t-\tau\vd}
         {\norm{\vb_t-\tau\vd}_2}.
\]

The step size $\tau$ is chosen so that the two query profiles have a prescribed Euclidean separation. If the desired query distance is $\gamma \in(0,2)$, then
\[
    \tau = \frac{\gamma}{\sqrt{4-\gamma^2}},
\]
which gives
\[
    \norm{\vx^+-\vx^-}_2 = \gamma
\]
before enforcing nonnegativity. To ensure that both profiles remain in the positive orthant, we also compute
\[
    \tau_{\max}
    =
    \min_{\ell \mid d_\ell\neq 0}
    \frac{b_{t,\ell}}{|d_\ell|},
\]
and use
\[
    \tau
    \leftarrow
    \min\left\{
        \frac{\gamma}
             {\sqrt{4-\gamma^2}},
        0.98\tau_{\max}
    \right\}.
\]
The factor $0.98$ is a fixed margin that prevents the query from reaching the boundary of the positive orthant.

The direction is chosen using the local curvature of the analytic center barrier. Let
\[
    H_t
    =
    \nabla^2\Phi_t(\bar{\vw}_t)
\]
denote the Hessian of the logarithmic barrier at the analytic center.
We restrict the query normal to the balanced subspace, $\left\{ \vd\in\R^\obj \mid  \inner{\vd}{\bar{\vw}_t}=0 \right\}$, so that the two profiles have equal utility under the current analytic center.
We choose $\vd$ to be the unit eigenvector corresponding to the smallest
eigenvalue of $H_t$. Equivalently,
\[
    \vd
    \in
    \argmin_{\norm{\vz}_2=1,\;
             \inner{\vz}{\bar{\vw}_t}=0}
    \vz^\top H_t\vz.
\]
The direction of smallest curvature corresponds to the direction of greatest uncertainty around the analytic center, so querying along this direction heuristically produces cuts that reduce the uncertainty of the remaining weight space.

%% file: text/appendix/appendix-sieving.tex
\section{Further details on iterative sieving}
\label{sec:appendix-sieving}
In this section, we provide further details on the iterative sieving method, including noise robustness.

The cutting plane method is effective in practice, but it does not provide a worst-case query complexity. When eliciting preferences from real users, we desire a guarantee that our algorithm terminates with the prescribed approximation accuracy. We therefore consider an algorithm with a provable expected query-complexity bound adapted from the \textit{iterative sieving} algorithm of~\citet{Cohen25:Combinatorial}. The algorithm operates on a finite cover of $\Objset$, and returns a point whose utility is approximately maximal under the user's linear utility.  In this subsection, we analyze the Euclidean-normalized representative of the user's preference direction. Throughout this subsection we write the true weight vector as
\[
    \vu^*\in\Objset \text{ and }
    \norm{\vu^*}_2=1.
\]
This is without loss since positive rescalings of a linear utility function induce the same ordering over outcomes. By Cauchy--Schwarz, the unique maximizer of $\vx\mapsto \inner{\vx}{\vu^*}$ over $\Objset$ is $\vx=\vu^*$ (\textit{i.e.}, the maximizing element is equivalent to the weight). Unlike ACCPM, which maintains a weight space over weights, the sieving algorithm maintains a finite set of candidate outcome profiles. Because the maximizer over $\Objset$ is exactly $\vu^*$, eliminating suboptimal profiles also eliminates candidate preference directions.

The algorithm relies on the \emph{envelope} and \emph{\CDfull} of a set, denoted $\envelope(\sigma)$ and $\CDmath(\cX)$ for a set $\cX$ and a sorted subset $\sigma$. The key idea of the envelope is to leverage an ordering of points induced by $\vu^*$, and remove all points that can be inferred as suboptimal from this ordering. The  \CDfull is a measure of how many points are needed to derive such inferences, and is similar in many ways to the \emph{inference dimension}~\citep{Kane17:Active}. Intuitively, once a subset $\cY$ has been sorted to obtain an ordering $\sigma$, some comparisons involving points outside $\sigma$ can be inferred without querying them directly. The envelope $\envelope(\sigma)$ encodes the cone of pairwise differences certified by the order $\sigma$, and $\CDmath(\cX)$ controls how large a random sorted subset must be before many points can be inferred to be suboptimal. We provide the necessary background from the work of~\citet{Cohen25:Combinatorial} in~\Cref{sec:appendix-cohen}. We present our iterative sieving procedure in~\Cref{alg:linear-elicitation}.

\begin{algorithm}[t]
\caption{Iterative sieving preference elicitation}
  \label{alg:linear-elicitation}
    $\cX \coloneqq$ an $\epsilon/\sqrt{2}$-cover of $\Objset$ \\
    $\vy \coloneqq \bot$, $k \coloneqq c \obj \log (1/\eps)$ for a constant $c$  \\ 
    \While{$|\cX|$ is more than $2k$}{
        Independently include each point of
        $\cX$ in a subset $\cY$ with probability $2k/|\cX|$ \\
        $\cY \gets \cY \cup \{ \vy\}$ \\
        $\sigma \gets$ result of sorting $\cY$ using $O(k \log k)$ expected comparisons
        \\
        $\vy \gets$ first (largest) element of $\sigma$ \\
        Eliminate all points
        $\vx \in \cX \setminus \{\vy\}$ such that
        $\dist(\vy - \vx, \envelope(\sigma)) < \eps^2/4$
  }
  \Return the maximum element, $\widehat{\vu}$, in $\cX$ found by a linear search
\end{algorithm}

Pairwise queries are asked during the sorting phase and the final search. We use the \textit{mergesort} algorithm for sorting. The algorithm guarantees that the returned weight is an $\eps$-approximation of the Euclidean-normalized weight vector.

\begin{restatable}[Weight approximation for~\Cref{alg:linear-elicitation}]{proposition}{SievingWeightApprox}
\label{lem:sieving-epsilon-approximation}
Let $\vu^*\in\Objset$ be the Euclidean-normalized true weight vector, and
let $\widehat{\vu}$ be the point returned by
\Cref{alg:linear-elicitation}. Then
\[
    \norm{\widehat{\vu}-\vu^*}_2\le \eps.
\]
\end{restatable}

Furthermore, the elicited weight vector approximates the utility for all profiles in
$\Objset$.

\begin{restatable}[Utility approximation for~\Cref{alg:linear-elicitation}]{corollary}{SievingUtilityApprox}
\label{cor:sieving-utility-approximation}
For all $\vx\in\Objset$,
\[
    \left|
        \inner{\vx}{\widehat{\vu}}
        -
        \inner{\vx}{\vu^*}
    \right|
    \le \eps.
\]
\end{restatable}

Finally, we can bound the query complexity of the algorithm.

\begin{restatable}[Query complexity of \Cref{alg:linear-elicitation}]{proposition}{SievingQueryComplexity}
\label{lem:sieving-query-complexity}
  The expected number of queries that are asked by \Cref{alg:linear-elicitation} is at most $O(\obj^2 \log^2(1/\epsilon) \log (\obj\log (1/\epsilon))))$.
\end{restatable}

The proof relies on the fact that the approximate conic dimension of $\cX$ is $O(\obj \log (1/\eps))$. The query complexity of~\Cref{alg:linear-elicitation} is larger than the information-theoretic lower bound in~\Cref{lem:linear-lower-bound}. However, in contrast to the heuristic cutting plane method, it provides an explicit approximation guarantee and an expected query-complexity bound. Another benefit of the sieving procedure is that it can be run on arbitrary finite point sets, not only on an $\eps$-cover. In that case, the guarantee becomes a maximization guarantee over the supplied point set. This is useful when the points correspond to medically meaningful profiles or to outcomes generated by policy simulations. Under noisy responses, the point removal step can introduce error that increases with $\delta$. We give a precise analysis in~\Cref{sec:appendix-sieving}, along with further details on the iterative sieving algorithm.

\subsection{Computing envelope membership}

Given $\sigma$ which is a sorted ordering of points $\vx_1,\ldots \vx_m$ where $\vx_1 = \vy$ is the maximizing value, we can determine if a point $\vx$ is dominated by $\vy$ according to an unknown weight vector $\vu^*$ using a \textit{linear program (LP)}. While we can theoretically allow for a margin of $\epsilon$ distance to the envelope, for computational efficiency we compute membership. Exact membership is a sufficient condition for the approximate envelope test and preserves the validity of every elimination. However, it may eliminate fewer points. We did not find this to be the case in our setting, and the same number of points was removed. This is likely due to  $\epsilon$ being near-zero. We also compute the membership with respect to a larger envelope than the one in~\Cref{def:envelope} by making use of the monotonicity. 

The computation is equivalent to determining whether $\vy-\vx$ lies in the cone generated by the envelope directions. We know due to the ordering of $\sigma$ that $\inner{\vx_i}{\vu^*} \ge \inner{\vx_{i+1}}{\vu^*} ~\forall ~i\in [1,m-1]$. Therefore, the difference vectors $ \vd_i = \vx_i - \vx_{i+1}$ must satisfy $\inner{\vd_i}{\vu^*} \ge 0$. We can also deduce that $\inner{\vx_1}{\vu^*} \ge \inner{\vx}{\vu^*} \iff \inner{\vy - \vx}{\vu^*} \ge 0$.

To determine if $\vy - \vx$ is in the envelope, we check if there exists a solution to the following linear program. Because we only care about the existence of a valid coefficient vector $\alpha = (\alpha_1, \ldots, \alpha_{m-1})$, we only care about finding a feasible solution.
$$
\begin{aligned}
\text{find} \quad & \alpha_1, \alpha_2, \ldots, \alpha_{m-1} \\
\text{subject to} \quad & \sum_{i=1}^{m-1} \alpha_i \vd_i \le \vy - \vx\\
& \alpha_i \ge 0, \quad \forall i \in [1, m-1]
\end{aligned}
$$

The linear program searches for a conic combination of the difference vectors that is bounded from above by the difference $\vy - \vx$. It is clear to see that if the linear program finds a feasible solution, it guarantees that for every nonnegative weight vector $\vu \ge \mathbf{0}$ satisfying the ordering constraints of $\sigma$, the utility of $\vy$ remains at least the utility of $\vx$. Consider that the LP returned a feasible $\alpha$ and let $\vu^*$ be the true weight vector. We know from the constraints of the LP that 
$$ 
    \sum_{i=1}^{m-1} \alpha_i \vd_i \le \vy - \vx.\\
$$
Taking the inner product with $\vu^*$ of both sides,
\begin{align*}
    \inner{\sum_{i=1}^{m-1} \alpha_i \vd_i }{\vu^*} & \le \inner{\vy - \vx}{\vu^*} \\
    \sum_{i=1}^{m-1} \alpha_i \inner{\vd_i}{\vu^*} & \le \inner{\vy - \vx}{\vu^*}.
\end{align*}
Since $\inner{\vd_i}{\vu^*} \ge 0$ and $\alpha_i \ge 0$, we get $0 \le \inner{\vy - \vx}{\vu^*}$. 

On the other hand, if the program is infeasible, it means that there exists at least one valid weight vector $\vu \ge \mathbf{0}$ where $\inner{\vx}{\vu} > \inner{\vy}{\vu}$.

\subsubsection{Noisy envelope}
\label{sec:noisy-envelope}

The envelope-based elimination rule is exact in the noiseless setting
because a sorted subsequence certifies a set of linear inequalities that the true weight vector must satisfy. In the noisy setting, these
certificates become approximate. As a result, conic combinations of many comparison constraints can amplify the comparison error.

Let
\[
    \sigma=(\vx_1,\ldots,\vx_m)
\]
be a subsequence sorted from largest to smallest according to the noisy
comparison oracle, and define
\[
    \vd_i=\vx_i-\vx_{i+1},
    \qquad i\in\{1,\ldots,m-1\}.
\]
In the noiseless setting, the ordering implies
\[
    \inner{\vd_i}{\vu^*}\ge 0
    \qquad
    \forall i\in\{1,\ldots,m-1\}.
\]
In the noisy setting, we assume that the sorted sequence satisfies
\[
    \inner{\vd_i}{\vu^*}\ge -\delta
    \qquad
    \forall i\in\{1,\ldots,m-1\}.
\]
Here, $\delta$ denotes the noise tolerance with respect to the Euclidean-normalized preference direction $\vu^*$. A point $\vx$ is eliminated in the noiseless setting when $\vy-\vx$ lies in, or is sufficiently close to, the envelope generated
by the sorted sequence. In the noisy setting, we search for nonnegative
coefficients $\alpha_1,\ldots,\alpha_{m-1}$ and a nonnegative slack
vector $\va\in\R^\obj_{\ge 0}$ such that
\[
    \vy-\vx
    =
    \sum_{i=1}^{m-1}\alpha_i \vd_i+\va.
\]
The coefficients $\alpha_i$ determine how much the noise can be amplified. Since each inequality $\inner{\vd_i}{\vu^*}\ge 0$ is relaxed to $\inner{\vd_i}{\vu^*}\ge -\delta$, a certificate with large $\norm{\alpha}_1$ may accumulate a large error.

For this reason, the noisy version of the envelope test should use bounded certificates. For a parameter $B\ge 0$, define the $B$-bounded envelope of $\sigma$ by
\[
    \envelope_B(\sigma)
    =
    \left\{
        \sum_{i=1}^{m-1}\alpha_i\vd_i+\va \mid
        \alpha_i\ge 0,\;
        \norm{\alpha}_1\le B,\;
        \va\in\R^\obj_{\ge 0}
    \right\}.
\]
The robust elimination rule removes $\vx$ only if
\[
    \dist(\vy-\vx,\envelope_B(\sigma))<\eps_{\mathrm{env}},
\]
where $\eps_{\mathrm{env}}$ is the tolerance of the approximate envelope.

The following proposition states the resulting utility degradation.

\begin{proposition}[Noisy envelope degradation]
\label{lem:noisy-envelope-degradation}
Suppose the sorted subsequence $\sigma=(\vx_1,\ldots,\vx_m)$ satisfies 
\[
    \inner{\vx_i-\vx_{i+1}}{\vu^*}\ge -\delta
    \qquad
    \forall i\in\{1,\ldots,m-1\}.
\]
Let $\vx$ be a point satisfying
\[
    \dist(\vy-\vx,\envelope_B(\sigma))
    \le
    \eps_{\mathrm{env}}.
\]
Then
\[
    \inner{\vx}{\vu^*}
    \le
    \inner{\vy}{\vu^*}
    +
    \eps_{\mathrm{env}}
    +
    B\delta.
\]
\end{proposition}

\begin{proof}
Since
\[
    \dist(\vy-\vx,\envelope_B(\sigma))
    \le
    \eps_{\mathrm{env}},
\]
there exist coefficients $\alpha_i\ge 0$ with
$\norm{\alpha}_1\le B$, a vector $\va\in\R^\obj_{\ge 0}$, and a residual
$r$ satisfying $\norm{r}_2\le \eps_{\mathrm{env}}$ such that
\[
    \vy-\vx
    =
    \sum_{i=1}^{m-1}\alpha_i\vd_i+\va+r.
\]
Taking the inner product with $\vu^*$ gives
\[
\begin{aligned}
    \inner{\vy-\vx}{\vu^*}
    &=
    \sum_{i=1}^{m-1}
        \alpha_i\inner{\vd_i}{\vu^*}
    +
    \inner{\va}{\vu^*}
    +
    \inner{r}{\vu^*}.
\end{aligned}
\]
We know from the assumption that
\[ \inner{\vd_i}{\vu^*}\ge-\delta \qquad \forall i. \]
In addition, since $\va\ge 0$ and $\vu^*\ge 0$,
\[
    \inner{\va}{\vu^*}\ge 0.
\]
Finally, since $\norm{\vu^*}_2=1$, Cauchy--Schwarz gives
\[
    \inner{r}{\vu^*}
    \ge
    -\norm{r}_2\norm{\vu^*}_2
    \ge
    -\eps_{\mathrm{env}}.
\]
Combining these bounds,
\[
    \inner{\vy-\vx}{\vu^*}
    \ge
    -\delta\sum_{i=1}^{m-1}\alpha_i
    -
    \eps_{\mathrm{env}}
    \ge
    -B\delta-\eps_{\mathrm{env}}.
\]
Rearranging gives
\[
    \inner{\vx}{\vu^*}
    \le
    \inner{\vy}{\vu^*}
    +
    \eps_{\mathrm{env}}
    +
    B\delta.
\]
\end{proof}

The parameter $B$ therefore controls the robustness of the noisy envelope rule. Without bounding $\norm{\alpha}_1$, an arbitrarily small comparison error can be amplified by a certificate with very large conic coefficients. Imposing $\norm{\alpha}_1\le B$ prevents this amplification and gives a utility degradation proportional to $B\delta$. 

This modification may reduce the number of points eliminated in each sieving round. The set $\envelope_B(\sigma)$ is smaller than the unrestricted envelope. In this case, the original query-complexity bound does not automatically apply to the noisy envelope. A query-complexity analysis could be obtained by replacing the approximate conic dimension with a $B$-restricted approximate conic dimension.

%% file: text/appendix/minimax-regret.tex
\section{Discussion of minimax regret}
\label{sec:appendix-minimax-regret}

An alternative approach to preference elicitation is based on \emph{minimax regret}~\citep{Boutilier04:Eliciting}. Given a weight space of feasible weight vectors $\cW_t$ and a feasible set of outcomes $\cP$, the maximum regret of an outcome
$\vx\in\cP$ is
\[
    \operatorname{MR}(\vx; \cW_t,\cP)
    =
    \max_{\vw\in\cW_t}
    \left[
        \max_{\vy\in\cP}\inner{\vy}{\vw}
        -
        \inner{\vx}{\vw}
    \right].
\]
Minimax regret selects an outcome that minimizes the worst-case utility loss and has been widely studied for decision support and preference elicitation~\citep{Boutilier06:Constraint}. Like ACCPM, the approach of~\citet{Boutilier04:Eliciting} provides an \emph{ex post} bound on the regret, but does not provide an \emph{ex ante} bound on the number of queries required to
reach a prescribed bound.

Our objective differs from the standard minimax-regret setting. Rather than selecting a single robust outcome, we seek to recover each stakeholder's utility vector so that the resulting utilities can be aggregated and reused across changing policy frontiers. A regret bound obtained for one feasible set does not generally transfer to another, whereas an $\epsilon$-approximation of the utility vector does.

Minimax regret can nevertheless be used to obtain such a weight-approximation guarantee if regret is defined over the domain $\Objset$. Let
\[
    \cU_t
    =
    \left\{
        \frac{\vw}{\norm{\vw}_2}
        ~\middle|~
        \vw\in\cW_t
    \right\}
\]
denote the space of Euclidean-normalized preference directions. For a candidate weight $\widehat{\vu}\in\Objset$, we define the maximum regret over $\Objset$ as
\[
\begin{aligned}
    \operatorname{MR}
    (\widehat{\vu};\cU_t)
    &=
    \max_{\vu\in\cU_t}
    \left[
        \max_{\vx\in\Objset}\inner{\vx}{\vu}
        -
        \inner{\widehat{\vu}}{\vu}
    \right] \\
    &=
    \max_{\vu\in\cU_t}
    \left[
        1-\inner{\widehat{\vu}}{\vu}
    \right].
\end{aligned}
\]
The second equality follows because
\[
    \max_{\vx\in\Objset}\inner{\vx}{\vu}=1
\]
for every $\vu\in\Objset$. Since both
$\widehat{\vu}$ and $\vu$ have unit Euclidean norm,
\[
    1-\inner{\widehat{\vu}}{\vu}
    =
    \frac{1}{2}
    \norm{\widehat{\vu}-\vu}_2^2.
\]
As a result, we obtain the equality
\[
    \operatorname{MR}
    (\widehat{\vu};\cU_t)
    =
    \frac{1}{2}
    \max_{\vu\in\cU_t}
    \norm{\widehat{\vu}-\vu}_2^2.
\]

Reducing the minimax regret over the sphere to $\epsilon^2/2$ yields the same weight-approximation guarantee considered in our other algorithms. In contrast, a small regret bound over a particular feasible policy set does not imply that the underlying utility vector has been accurately learned.

%% file: text/appendix/appendix-aggregation.tex
\section{One method for aggregating linear preferences}
\label{sec:appendix-weight-aggregation}

After eliciting individual utility functions, we aggregate stakeholder preferences into a single community-level utility function. Let $m$ denote the number of stakeholders, and let $\widehat{\vu}^{(i)}\in\Objset$ be the Euclidean-normalized preference direction elicited from stakeholder $i\in\{1,\ldots,m\}$. We then define the aggregate stakeholder weight vector by the arithmetic mean
\[
    \bar{\vu}
    =
    \frac{1}{m}
    \sum_{i=1}^m
    \widehat{\vu}^{(i)}.
\]

This aggregation rule has a simple welfare interpretation. If stakeholder $i$ has linear utility
$
    \widehat{\ut}^{(i)}(\vx)
    =
    \inner{\vx}{\widehat{\vu}^{(i)}},
$
then the average stakeholder welfare of an outcome $\vx$ is
\[
    \frac{1}{m}\sum_{i=1}^m
    \widehat{\ut}^{(i)}(\vx)
    =
    \frac{1}{m}
    \sum_{i=1}^m
    \inner{\vx}{\widehat{\vu}^{(i)}}
    =
    \inner{\vx}{\bar{\vu}}.
\]
So, optimizing the aggregate utility function $\bar{\ut}(\vx) = \inner{\vx}{\bar{\vu}}$ is equivalent to maximizing utilitarian social welfare over the elicited stakeholder utilities. The resulting vector $\bar{\vu}$ can therefore be used as the objective for downstream allocation-policy optimization. 

The $\eps$-approximation guarantee on the individual utilities also transfers under aggregation. If each elicited weight vector satisfies $\norm{\widehat{\vu}^{(i)} - \vu^{*(i)}}_2 \le \epsilon$, then $\norm{\bar{\vu}-\bar{\vu}^*}_2 \le \eps$. And, for every $\vx\in\Objset$, $\left| \inner{\vx}{\bar{\vu}} - \inner{\vx}{\bar{\vu}^*} \right| \le \epsilon$. 

Other aggregation rules, such as fairness-aware or strategy-proof social choice rules, could also be leveraged. The aggregation requires fixing a common normalization, and one could perform an analogous aggregation using exclusively simplex normalized weights.

%% file: text/appendix/appendix-proofs.tex
\section{Omitted proofs}
\label{sec:appendix-proofs}

We include the formal proofs of all claims in the paper. We begin with the proof of~\Cref{thm:general-utility-lower-bound}.

\GeneralUtilityLowerBound*

\begin{proof}
We write
\[
    \cF
    =
    \left\{
        S\subseteq [\obj] : |S|= d/2 
    \right\},
\]
so that
\[
    |\cF|=\binom{\obj}{\obj/2}.
\]
For each $S\in\cF$, if $\mathbf{1}_S\in\{0,1\}^{\obj}$ denotes the indicator vector of $S$, we define
\[
    \mathbf{p}_S
    =
    \frac{\mathbf{1}_S}{\sqrt{d/2}} \in \Objset.
\]
Now, for every family $A\subseteq\cF$, we define a utility function $\ut_A:\R^\obj_{\ge0}\to\{0,1\}$ by
\[
    \ut_A(\vx)=1
    \quad\Longleftrightarrow\quad
    \exists S\in A
    \text{ s.t. }
    \vx\ge \mathbf{p}_S
    \text{ coordinate-wise}.
\]
Each $\ut_A$ is coordinate-wise monotone: if $\vx\le \vy$ coordinate-wise and $\ut_A(\vx)=1$, then $\vx\ge \mathbf{p}_S$ for some $S\in A$, so $\vy\ge\mathbf{p}_S$ and $\ut_A(\vy)=1$.

We next observe that, on $\Objset$, the only points at which $\ut_A$ equals $1$ are the threshold points themselves. Specifically, for any $\vx\in\Objset$,
\[
    \ut_A(\vx)=1
    \quad\Longleftrightarrow\quad
    \vx=\mathbf{p}_S
    \text{ for some } S\in A.
\]
The reverse implication is immediate. For the forward implication, suppose $\ut_A(\vx)=1$. Then there exists $S\in A$ such that $\vx\ge \mathbf{p}_S$ coordinate-wise. Write $\vx=\mathbf{p}_S+\vz$. For $\vz\in\R^\obj_{\ge0}$, since $\mathbf{p}_S,\vz\ge0$ coordinate-wise,
\[
    \norm{\vx}_2^2
    =
    \norm{\mathbf{p}_S+\vz}_2^2
    =
    \norm{\mathbf{p}_S}_2^2
    +
    2\inner{\mathbf{p}_S}{\vz}
    +
    \norm{\vz}_2^2.
\]
Both $\vx$ and $\mathbf{p}_S$ lie in $\Objset$, so
\[
    \norm{\vx}_2^2
    =
    \norm{\mathbf{p}_S}_2^2
    =
    1.
\]
Thus,
\[
    2\inner{\mathbf{p}_S}{\vz}
    +
    \norm{\vz}_2^2
    =
    0.
\]
Since both terms are nonnegative, $\vz=\mathbf{0}$, and hence $\vx=\mathbf{p}_S$.

We now reduce from two-party set disjointness on the universe $\cF$. Alice receives a subset $A\subseteq\cF$, Bob receives a subset $B\subseteq\cF$, and they must decide whether
\[
    A\cap B=\varnothing.
\]
The universe size is
\[
    N=|\cF|=\binom{\obj}{\obj/2}.
\]
A standard fact is that set disjointness on a universe of size $N$ requires at least $\Omega(N)$ bits of communication to solve with probability at least $2/3$~\citep{Kalyanasundaram92:Probabilistic,Razborov90:Distributional}.

Given a set-disjointness instance $(A,B)$, we construct a welfare-maximization instance by setting $\ut_1=\ut_A$ and $\ut_2=\ut_B$. For any $\vx\in\Objset$, the characterization above implies
\[
    \ut_A(\vx)=\ut_B(\vx)=1
    \quad\Longleftrightarrow\quad
    \vx=\mathbf{p}_S
    \text{ for some } S\in A\cap B.
\]
Therefore,
\[
    \max_{\vx\in\Objset}
    \left(
        \ut_A(\vx)+\ut_B(\vx)
    \right)
    =
    \begin{cases}
        2, & \text{if } A\cap B\ne\varnothing,\\
        \le 1, & \text{if } A\cap B=\varnothing.
    \end{cases}
\]

Now, if there is a protocol using $c$ bits of communication that always outputs an exact maximizer of $\ut_1(\vx)+\ut_2(\vx)$ over $\Objset$ for every pair of monotone utilities, Alice and Bob can solve set disjointness as follows. They run the welfare-maximization protocol on the utilities $\ut_A$ and $\ut_B$ and obtain an optimal point $\widehat{\vx}\in\Objset$. Alice then sends Bob one additional bit, namely $\ut_A(\widehat{\vx})$. Bob computes $\ut_B(\widehat{\vx})$ and declares
\[
    A\cap B\ne\varnothing
    \quad\Longleftrightarrow\quad
    \ut_A(\widehat{\vx})=\ut_B(\widehat{\vx})=1.
\]
This decision rule is correct. If $A\cap B\ne\varnothing$, then the optimal welfare is $2$, so every exact optimizer $\widehat{\vx}$ satisfies
\[
    \ut_A(\widehat{\vx})+\ut_B(\widehat{\vx})=2,
\]
and hence both utilities equal $1$. If $A\cap B=\varnothing$, then no point in $\Objset$ can make both utilities equal $1$.

Thus, a $c$-bit exact welfare-maximization protocol gives a $(c+1)$-bit protocol for set disjointness on a universe of size $N$. Since deterministic set disjointness requires at least $\Omega(N)$ bits, we have $c \geq 2^{\Omega(d)}$.

Finally, a ternary comparison-based elicitation protocol with $Q$ queries has at most $3^Q$ possible response transcripts, and hence communicates at most $Q\log_2 3$ bits of information through the responses. Therefore, any such protocol that always solves the exact welfare-maximization task must also use
\[
    Q
    =
    2^{\Omega(\obj)}
\]
queries in the worst case.
\end{proof}

As a result, any ternary comparison-based elicitation protocol that solves this task also requires exponentially many queries in the worst case. This lower bound holds even for a restricted subclass of monotone utilities: binary-valued threshold utilities over a finite subset of $\Objset$. Monotonicity alone is not sufficient to obtain query-efficient guarantees. The linear model studied in the rest of the paper makes preference elicitation tractable and provides interpretability through objective weights.

We move on to the proof of~\Cref{lem:linear-lower-bound}.

\begin{restatable}[Lower bound on linear query complexity]{proposition}{LinearLowerBound}
\label{lem:linear-lower-bound}
  Any deterministic or zero-error randomized algorithm asking pairwise comparisons of profiles in $\Objset$ that outputs a weight $\widehat{\vw}$ such that $\sup_{\vx \in \Objset} | \inner{\vx}{\vw^*} - \inner{\vx}{\hvw} | \le \epsilon$ for every $\vw^*\in\cW$ must in the worst case ask $\Omega( \obj \log (1/ \epsilon))$ queries. 
\end{restatable}

\begin{proof}
Let
\[
    D(\vw_1,\vw_2)
    =
    \sup_{\vx\in\Objset}
    \left|
        \inner{\vx}{\vw_1-\vw_2}
    \right|
\]
denote the induced utility-error metric. We first observe that $D$
dominates Euclidean distance up to a constant on $\cW$. Let
$\vz=\vw_1-\vw_2$. Since $\vw_1,\vw_2\in\cW$, we have
$\sum_{\ell=1}^{\obj} z_\ell=0$. Let $\vz_+$ and $\vz_-$ denote the
positive and negative parts of $\vz$. Then
\[
    \max\{\norm{\vz_+}_2,\norm{\vz_-}_2\}
    \ge
    \frac{\norm{\vz}_2}{\sqrt{2}}.
\]
Since the normalized positive or negative part belongs to $\Objset$, it
follows that
\[
    D(\vw_1,\vw_2)
    \ge
    \frac{1}{\sqrt{2}}\norm{\vw_1-\vw_2}_2 .
\]
Therefore, after changing constants, any Euclidean packing of the simplex
also gives a packing under $D$.

Now construct a $2\epsilon$-separated packing set
$\cK\subset\cW$ such that for any distinct
$\vw_1,\vw_2\in\cK$,
\[
    D(\vw_1,\vw_2)>2\epsilon.
\]
Since $\cW$ is a $(\obj-1)$-dimensional simplex, there exists such a
packing with
\[
    |\cK|\ge c(1/\epsilon)^{\obj-1}
\]
for a constant $c>0$.

The following transcript-counting argument is immediate for deterministic algorithms. It also applies to randomized algorithms that succeed with probability one for every $\vw^*\in\cW$. Since the packing set $\cK$ is finite, there exists a realization of the algorithm's randomness under which the algorithm succeeds for every $\vw^*\in\cK$. We can fix such a realization, making the algorithm deterministic on $\cK$.

Each pairwise comparison yields at most three outcomes:
$\succ$, $\prec$, and $\approx$. Hence, an adaptive algorithm asking
$Q$ queries can induce at most $3^Q$ distinct transcripts. If
\[
    3^Q < |\cK|,
\]
then by the pigeonhole principle there must exist two distinct target
weights $\vw_1,\vw_2\in\cK$ that lead to the same transcript. The
algorithm must therefore output the same estimate $\widehat{\vw}$ for
both possible true weights.

By the triangle inequality for $D$,
\[
    2\epsilon
    <
    D(\vw_1,\vw_2)
    \le
    D(\vw_1,\widehat{\vw})
    +
    D(\vw_2,\widehat{\vw}).
\]
Thus,
\[
    \max\{
        D(\vw_1,\widehat{\vw}),
        D(\vw_2,\widehat{\vw})
    \}
    >
    \epsilon.
\]
Therefore, the algorithm fails to achieve error at most $\epsilon$ for
at least one of the two possible true weights. Consequently, any
algorithm that succeeds for all $\vw^*\in\cW$ must satisfy
\[
    3^Q\ge |\cK|.
\]
Taking logarithms gives
\[
    Q
    \ge
    \log_3 |\cK|
    =
    \Omega((\obj-1)\log(1/\epsilon)).
\]
\end{proof}

We continue with the proofs of correctness for the ACCPM algorithm.

\AccpmWeightApprox*

\begin{proof}
The true weight vector $\vw^*$ satisfies all generated pairwise constraints across all iterations. Therefore, $\vw^*$ must be contained within $\cW_t$. The final estimate $\widehat{\vw}$ is in the same set. It follows from the termination condition that  $\norm{\vw^* - \widehat{\vw}}_2 \le \epsilon$.
\end{proof}

And the utility function approximation guarantee follows.

\AccpmUtilityApprox*

\begin{proof}
For any $\vx \in \Objset$ we can write $|\inner{\vx}{\vw^*} - \inner{\vx}{\hvw}| = |\inner{\vx}{\vw^* - \hvw}|.$ By Cauchy-Schwarz, 
$$|\inner{\vx}{\vw^* - \hvw}| \le \norm{\vw^* - \widehat{\vw}}_2 \cdot \norm{\vx}_2.$$

Since $\norm{\vx}_2 = 1$, we can apply~\Cref{lem:accpm-weight-approximation} to complete the proof.

\end{proof}

Next, we prove the correctness of the iterative sieving algorithm.

\SievingWeightApprox*

\begin{proof}
Let $\vx_u\in\cX$ be the point closest to $\vu^*$. Since $\cX$ is an
$\eps/\sqrt{2}$-cover of $\Objset$,
\[
    \norm{\vx_u-\vu^*}_2
    \le
    \frac{\eps}{\sqrt{2}}.
\]
Because both $\vx_u$ and $\vu^*$ lie on the unit sphere,
\[
    \norm{\vx_u-\vu^*}_2^2
    =
    2-2\inner{\vx_u}{\vu^*}.
\]
Therefore,
\[
    1-\inner{\vx_u}{\vu^*}
    =
    \frac{1}{2}\norm{\vx_u-\vu^*}_2^2
    \le
    \frac{\eps^2}{4}.
\]
Since $\inner{\vu^*}{\vu^*}=1$, this implies
\[
    \inner{\vu^*}{\vu^*}
    -
    \inner{\vx_u}{\vu^*}
    \le
    \frac{\eps^2}{4}.
\]

By applying \Cref{cor:epsilon-query} with 
$\eps^2/4$, the returned point $\widehat{\vu}$ satisfies
\[
    \inner{\widehat{\vu}}{\vu^*}
    \ge
    \inner{\vx_u}{\vu^*}
    -
    \frac{\eps^2}{4}.
\]
Therefore,
\[
    \inner{\vu^*}{\vu^*}
    -
    \inner{\widehat{\vu}}{\vu^*}
    \le
    \frac{\eps^2}{2}.
\]
Since $\widehat{\vu},\vu^*\in\Objset$,
\[
    \norm{\widehat{\vu}-\vu^*}_2^2
    =
    2-2\inner{\widehat{\vu}}{\vu^*}
    \le
    \eps^2.
\]
Taking square roots gives
\[
    \norm{\widehat{\vu}-\vu^*}_2\le \eps.
\]
\end{proof}

The proof of~\Cref{cor:sieving-utility-approximation} follows from the application of the Cauchy--Schwarz inequality.

\SievingUtilityApprox*

\begin{proof}
Let $\vx\in\Objset$. Then
\[
    \left|
        \inner{\vx}{\widehat{\vu}}
        -
        \inner{\vx}{\vu^*}
    \right|
    =
    \left|
        \inner{\vx}{\widehat{\vu}-\vu^*}
    \right|.
\]
By Cauchy--Schwarz,
\[
    \left|
        \inner{\vx}{\widehat{\vu}-\vu^*}
    \right|
    \le
    \norm{\vx}_2\norm{\widehat{\vu}-\vu^*}_2.
\]
Since $\vx\in\Objset$, $\norm{\vx}_2=1$, and the result follows from
\Cref{lem:sieving-epsilon-approximation}.
\end{proof}

We move on to the proof of the query complexity of~\Cref{alg:linear-elicitation}, leveraging the results of~\citet{Cohen25:Combinatorial}.

\SievingQueryComplexity*

\begin{proof}
    The size of $\cX$ is $|\cX| \le  C\left(1/\eps\right)^\obj$. Taking the logarithm yields $\log |\cX| = O(\obj \log (1/\eps))$. From \Cref{lem:approx-conic-dim-upperbound}, we know that $\CDmath_{\eps^2/4}(\cX) = O(\obj \log (1/\eps))$ as $\cX \subset \BB^\obj$ (denoting the $d$-dimensional unit ball). The result then follows from \Cref{cor:epsilon-query}.
\end{proof}

Finally, we prove the correctness of the hybrid algorithm. 

\HybridWeightApprox*

\begin{proof}
We consider the two possible ways the algorithm can terminate. First suppose that the cutting plane phase certifies $\diam_2(\cW_T) \le \frac{\eps}{\sqrt{\obj}}$.
The algorithm chooses $\bar{\vw}_T\in\cW_T$ and returns $\widehat{\vu} = \frac{\bar{\vw}}{\norm{\bar{\vw}}_2}$.
Let $\frac{\vu^*}{\norm{\vu^*}_1} = \vw^*$. Since $\vw^*\in\cW_T$, the diameter condition gives
\[
    \norm{\bar{\vw}_T-\vw^*}_2
    \le
    \frac{\eps}{\sqrt{\obj}}.
\]

For any $\vw,\vw'\in\cW$, Euclidean normalization satisfies
\[
    \left\|
        \frac{\vw}{\norm{\vw}_2}
        -
        \frac{\vw'}{\norm{\vw'}_2}
    \right\|_2
    \le
    \sqrt{\obj}\norm{\vw-\vw'}_2.
\]
Applying this with $\vw=\bar{\vw}_T$ and $\vw'=\vw^*$ gives the final bound. 

Now suppose the algorithm proceeds to the sieving phase. Before entering the sieving loop, the algorithm initializes
\[
    \vy
    =
    \argmax_{\vx\in\cX}\inner{\vx}{\bar{\vw}_T}
\]
and removes any point $\vx$ satisfying
\[
    \max_{\vw\in\cW_T}\inner{\vx-\vy}{\vw}\le 0.
\]
This pre-sieve elimination is safe since $\vw^*\in\cW_T$, and any
removed point satisfies
\[
    \inner{\vx-\vy}{\vw^*}\le 0.
\]
Because $\vu^*=\norm{\vu^*}_1\vw^*$, this implies
\[
    \inner{\vx-\vy}{\vu^*}
    =
    \norm{\vu^*}_1\inner{\vx-\vy}{\vw^*}
    \le 0.
\]
So the pre-sieve pruning step never removes a point that is strictly
better than the initialized incumbent $\vy$.

During the sieving iterations, the first elimination rule,  \Circled{1}, is exactly the ordinary sieving elimination rule. To complete the proof, it is sufficient to show that the additional elimination rule, \Circled{2}, is safe.

Consider one sieving iteration with sorted sequence
\[
    \sigma=(\vx_1,\ldots,\vx_m),
\]
ordered from largest to smallest, and let
\[
    \vy=\vx_1.
\]
The algorithm defines
\[
    \cW_\sigma
    =
    \left\{
        \vw\in\cW_T:
        \inner{\vx_i-\vx_{i+1}}{\vw}\ge 0
        \quad
        \forall i=1,\ldots,m-1
    \right\}.
\]
Because the sorted order is correct under the true preference direction $\vu^*$, we have
\[
    \inner{\vx_i-\vx_{i+1}}{\vu^*}\ge 0
    \qquad
    \forall i=1,\ldots,m-1.
\]
Since $\vw^*=\vu^*/\norm{\vu^*}_1$ is a positive rescaling of $\vu^*$,
this also implies
\[
    \inner{\vx_i-\vx_{i+1}}{\vw^*}\ge 0
    \qquad
    \forall i=1,\ldots,m-1.
\]
Together with $\vw^*\in\cW_T$, we obtain
\[
    \vw^*\in\cW_\sigma.
\]

Now suppose a candidate $\vx$ satisfies the second elimination rule:
\[
    \max_{\vw\in\cW_\sigma}
    \inner{\vx-\vy}{\vw}
    \le 0.
\]
Since $\vw^*$ is feasible for this maximization problem,
\[
    \inner{\vx-\vy}{\vw^*}\le 0.
\]
Using $\vu^*=\norm{\vu^*}_1\vw^*$ and
$\norm{\vu^*}_1>0$, we get
\[
    \inner{\vx-\vy}{\vu^*}
    =
    \norm{\vu^*}_1
    \inner{\vx-\vy}{\vw^*}
    \le 0.
\]
So the second elimination rule never removes a point that is strictly better than the current sampled maximum $\vy$ under the true Euclidean-normalized utility. Therefore, the hybrid elimination rule is at least as safe as the ordinary sieving elimination rule. The result then follows from~\Cref{lem:sieving-epsilon-approximation}.
\end{proof}

The proofs of~\Cref{cor:hybrid-utility-approximation} and~\Cref{cor:hybrid-query-complexity} follow immediately from~\Cref{cor:sieving-utility-approximation} and~\Cref{lem:sieving-query-complexity} respectively. 

%% file: text/appendix/appendix-heart-transplant-elicitation.tex
\section{Further details on US heart transplantation objectives}
\label{sec:appendix-objectives}

We provide additional details on the six objectives.

\begin{enumerate}
    \item \textit{Average life years gained per donor}: The average projected additional years of life a recipient will gain from a transplant compared to remaining on the waitlist averaged across all donor hearts. For example, if a patient would live only one year on the waitlist without a transplant, but receives a heart and lives 11 years, the life years gained are 10 years. The average projected additional years captures how effectively each donor heart is being used.

    \item \textit{One year waitlist survival rate}: The percentage of waitlisted patients who do not receive a transplant that survive one year after being listed. The survival rate among non-transplanted patients reflects how the most medically urgent candidates are prioritized to prevent death on the waitlist. However, reducing mortality on the waitlist does not always equate to better post-transplant outcomes; a sick patient receiving a transplant may live only a short time following the operation, which is reflected by a low additional years of life of the recipient.

    \item \textit{One year waitlist transplant rate}: The percentage of waitlisted patients who receive a heart transplant within one year of being listed. The transplant rate among patients within one year of listing reflects how efficient the system is at reducing wait times. A high one year waitlist transplant rate does not capture the effectiveness of those transplants, but rather prioritizes minimizing the time patients have to wait.

    \item \textit{Socio-economically disadvantaged patient transplant rate}: The percentage of waitlisted patients who are from a disadvantaged community that receive a transplant. A disadvantaged community is determined by a \textit{distressed community index (DCI)} $\ge 80$. The transplant rate among this community is a measure of fairness of the system, and a higher value prioritizes socio-economically disadvantaged patients with respect to the rest of the patient population.

    \item \textit{Biologically disadvantaged patient transplant rate}: The percentage of waitlisted patients who are biologically disadvantaged that receive a transplant. A biologically disadvantaged patient is defined as a patient with \textit{calculated panel reactive antibody (cPRA)} $> 50\%$ or blood type O. Due to biological factors, these patients are harder to match to donor hearts due to constraints on the donor biological characteristics such as blood type. The transplant rate among this community is a measure of fairness of the system, and a higher value prioritizes biologically disadvantaged patients with respect to the rest of the patient population.

    \item \textit{Average miles per donor}: The average distance a donor heart travels to the recipient in miles. A distance less than 100 miles typically uses ground transport, whereas greater distances rely on air transport. The distance is a measure of system efficiency and risk. The larger the transport distance, the more patients can potentially benefit from a donor heart, but at a higher risk of organ wastage or graft failure due to the limited ischemic time of hearts ($<4$ hours).
\end{enumerate}

%% file: text/appendix/appendix-figs-tables.tex
\section{Omitted tables and figures}

\ifthenelse{\boolean{isSingleColumn}}{%
        \renewcommand{\imgwidth}{0.7\linewidth}
    }{%
        \renewcommand{\imgwidth}{\linewidth}
    }%

\begin{figure}[h!]
    \centering
    \begin{subfigure}[b]{0.49\linewidth}
        \centering
        \includegraphics[width=\imgwidth]{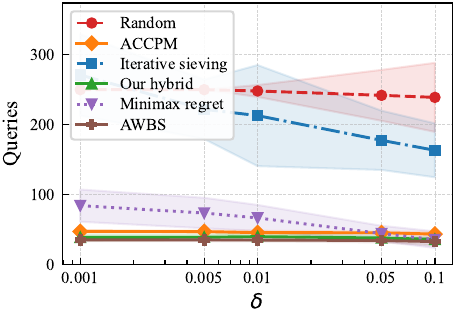}
        \caption{Queries per algorithm.}
        \label{fig:noisy_queries}
    \end{subfigure}
    \hfill
    \begin{subfigure}[b]{0.49\linewidth}
        \centering
        \includegraphics[width=\imgwidth]{figures/noisy_oracle_l2_error.pdf}
        \caption{$\ell_2$ error per algorithm.}
        \label{fig:noisy_l2}
    \end{subfigure}
    \caption{Results with a noisy oracle parameterized by $\delta$, using random weight vectors with a max budget of 250 queries. Shaded area shows standard deviation across 20 trials per $\delta$.}
    \label{fig:noisy_results_appendix}
\end{figure}

\begin{table}[h!]
    \centering
    \small
    \begin{tabular}{ccccc}
        \toprule
        \textbf{Sample size} & \textbf{Mean queries} & \textbf{Std queries} \\
        \midrule
        22 & 8.3 & 11.2 \\
        \bottomrule
    \end{tabular}
    \caption{Distribution of queries asked among users who did not complete the elicitation.}
    \label{tab:human-dropout}
\end{table}

%We also provide a more detailed distribution of the elicited weight vectors in~\Cref{tab:human-weight-distribution}.

\begin{table}[h]
\centering
\begin{tabular}{llr}
\toprule
Characteristic & Category & $N$ (\%) \\
\midrule
Age & 18-25 & 8 (30.8\%) \\
 & 26-35 & 8 (30.8\%) \\
 & 56+ & 6 (23.1\%) \\
 & 36-45 & 2 (7.7\%) \\
 & 46-55 & 2 (7.7\%) \\
\midrule
Education & Doctorate & 14 (53.8\%) \\
 & Bachelor's & 6 (23.1\%) \\
 & Master's & 5 (19.2\%) \\
 & Some College & 1 (3.8\%) \\
\midrule
Race / ethnicity & White & 18 (69.2\%) \\
 & Asian & 5 (19.2\%) \\
 & Hispanic or Latino, White & 1 (3.8\%) \\
 & Other & 1 (3.8\%) \\
 & White, Middle Eastern or North African & 1 (3.8\%) \\
\midrule
US region & Northeast & 15 (57.7\%) \\
 & West & 5 (19.2\%) \\
 & Midwest & 3 (11.5\%) \\
 & Not in US & 2 (7.7\%) \\
 & South & 1 (3.8\%) \\
\bottomrule
\end{tabular}
\caption{Demographics of survey participants.}
\label{tab:demographics}
\end{table}

%% file: text/appendix/appendix-hyperparameters.tex
\section{Hyperparameter settings and compute details for preference elicitation}
\label{sec:appendix-hyperparameters}

We detail the relevant hyperparameter settings used in experiments and for the final hybrid algorithm powering the preference elicitation with real users in~\Cref{tab:hybrid-hyperparameters}. The hyperparameters were set following a tuning phase in simulation. All simulated experiments are conducted on an M4 Pro processor with 24GB unified memory. For trial $i$ in the experiments, we utilize a random seed of $i \times 31 + 7$ to generate the weight vector and $i \times 97 + 13$ for the noisy oracle. The iterative sieving algorithm uses a fixed random seed of $1$. For the human experiments, we leverage a Linux cluster equipped with dual AMD EPYC 7252 CPUs (16 physical cores, 32 hardware threads) and 512 GB of RAM. 

\begin{table}[h!]
\centering
\begin{tabular}{@{}lcc@{}}
\toprule
\textbf{Hyperparameter} & \textbf{Symbol} & \textbf{Value} \\
\midrule

Target Euclidean accuracy
& $\eps$
& 0.05 \\

ACCPM query budget
& $T$
& 30 \\

ACCPM query separation
& $\gamma$
& 0.3 \\

ACCPM constraint slack
& $\delta$
& 0.01 \\

Iterative sieving sampling constant
& $c$
& 0.5 \\

Iterative sieving $\norm{\alpha}_1$ bound
& $B$
& 4.0 \\

\bottomrule
\end{tabular}
\caption{Hyperparameters for the hybrid cutting plane and sieving algorithm.}
\label{tab:hybrid-hyperparameters}
\end{table}

%% file: text/appendix/appendix-policy-optimization.tex
\section{Further details on policy optimization}
\label{sec:appendix-policy-optimization}

The simulation of historical data models the real dynamics of donor arrivals, patient waitlist additions and removals, and the progression of patient conditions (\textit{e.g.}, updating lab results). Standard constraints are enforced in simulation such as compatibility constraints (the donor and patient must have compatible blood types) and geographic constraints (the donor and patient must be located within 1,000 nautical miles of each other). We restrict the simulation to adult transplantation (age 18+). The simulator experiments are conducted on an M4 Pro processor with 24GB unified memory.

The simulation relies on estimates of the life years gained (for counterfactual allocations). As is standard in organ allocation works, we estimate the life years gained with Cox proportional hazards models~\citep{Cox72:Regression}. The Cox estimators are fitted using data from 1987-2022 from the UNOS patient registry. The training set contains $60,055$ examples of transplant outcomes and $120,282$ examples of waitlist survival times. We utilize 120 covariates of a patient-donor pair for graft survival prediction, and 30 covariates of a patient for waitlist survival prediction.

For policy optimization, we train and evaluate using data from 2019. The training set is from January to March 2019, and is the data for which we perform policy optimization. All policies are then evaluated on the real, unseen data from April to December 2019.  

\subsection{Potential-based allocation policy}

We build on the notion of \textit{potentials} for our policy optimization. Potentials were first introduced by~\citet{Dickerson12:Dynamic} for kidney exchange. Potentials were originally introduced to optimize the \textit{population life years gained} (\textit{i.e.}, the first of our six objectives). A potential-based allocation policy assigns each patient on the current waitlist an additive score of their predicted life years gained and their potential value. The predicted life years gained quantifies the immediate utility of transplantation, while the potential value captures the long-term value of maintaining certain patients in the pool. We use potentials differently; since we are not just maximizing the population life years, but instead balancing multiple objectives, we use the potentials to learn a corrective term that accounts for the tradeoff between the life years gained objective and the other five objectives.  

To learn the potentials, we leverage the newer, more scalable framework from~\citet{Zilberstein26:Learning}.  Their method uses self-supervised imitation learning to fit the potential function. They compute, using a hindsight optimal integer program (\textit{i.e.}, the supervisor), the optimal (for life years gained) allocation over the historical training set. The potentials, which are used online with unseen data,  are trained to mimic the decision of the supervisor. We adapt the original framework by swapping the optimization objective from the utilitarian population life years gained to the elicited utility function from users. We use the Gurobi Optimizer version 12.0.3~\citep{Gurobi26:Gurobi} for the supervising integer program. Our potential function uses a \textit{multi-layer perceptron (MLP)}, and is trained using a listwise ListNet loss function (\textit{e.g.}, minimize the \textit{Kullback-Leibler (KL)} divergence between the predicted softmax distribution over patients on the waitlist and the ground-truth distribution of what the supervisor selected). We include the relevant information for the MLP in~\Cref{tab:hyperparams-nn}. Our architecture is based on the original model used in~\citet{Zilberstein26:Near}, and we refer the reader to that paper for further details. 

\begin{table}[hbtp!]
    \centering
    \small
    \begin{tabular}{llc}
    \toprule
    \textbf{Category} & \textbf{Hyperparameter} & \textbf{Value} \\
    \midrule
    \textbf{Architecture} & Hidden layers & (128, 64, 32) \\
     & Activation function & Leaky ReLU \\
     & Leak coefficient & $1 \times 10^{-2}$ \\
     & Output activation & Identity \\
     & Input features & 11 \\
    \midrule
    \textbf{Training} & Optimizer & Adam \\
     & Loss function & Listwise ListNet \\
     & Learning rate & $1 \times 10^{-4}$ \\
     & Batch size & 32 \\
     & Epochs & 25 \\
     & Gradient clipping (norm) & 1.0 \\
     & L2 regularization & $1 \times 10^{-3}$ \\
     & Dropout rate & 0.3 \\
    \bottomrule
    \end{tabular}
    \caption{Hyperparameters for the neural network potential function.}
    \label{tab:hyperparams-nn}
\end{table}

\subsection{Greedy allocation policy}

The greedy algorithm is a baseline for maximizing the life years gained in isolation from the other objectives. The greedy policy selects the patient with the largest predicted life years gained to transplant at each donor arrival.

\subsection{\textit{Status quo} policy}

The \textit{status quo} policy in operation for adult heart transplants in the US uses a rule-based algorithm to assign patients into one of six tiers~\citep{Kilic21:Evolving}. 
The tier-determination is based on factors such as medical urgency, blood type, and geographic proximity. Within a tier, patients are prioritized based on wait time. When a donor is available, the allocation follows a lexicographic order of the tiers. The \textit{status quo} policy is agnostic to any utility function and was hand-crafted by domain experts.

We include the full tier list in~\Cref{tab:priority_tiers}.  There are 68 total priority tiers (tier 1 is the highest priority).  

\begin{table*}[htbp]
\footnotesize
\centering
\caption{\emph{Status quo} priority tiers.}
\label{tab:priority_tiers}
\begin{tabular}{clll@{\hspace{1em}}|@{\hspace{1em}}clll}
\toprule
\textbf{Tier} & \textbf{Status} & \textbf{Blood match} & \textbf{Distance (nm)} & \textbf{Tier} & \textbf{Status} & \textbf{Blood match} & \textbf{Distance (nm)} \\
\midrule
1 & 1 & Primary & $\leq 500$ & 35 & 2 & Primary & $\leq 2500$ \\
2 & 1 & Secondary & $\leq 500$ & 36 & 2 & Secondary & $\leq 2500$ \\
3 & 2 & Primary & $\leq 500$ & 37 & 3 & Primary & $\leq 2500$ \\
4 & 2 & Secondary & $\leq 500$ & 38 & 3 & Secondary & $\leq 2500$ \\
5 & 3 & Primary & $\leq 250$ & 39 & 4 & Primary & $\leq 1000$ \\
6 & 3 & Secondary & $\leq 250$ & 40 & 4 & Secondary & $\leq 1000$ \\
7 & 1 & Primary & $\leq 1000$ & 41 & 5 & Primary & $\leq 1000$ \\
8 & 1 & Secondary & $\leq 1000$ & 42 & 5 & Secondary & $\leq 1000$ \\
9 & 2 & Primary & $\leq 1000$ & 43 & 6 & Primary & $\leq 1000$ \\
10 & 2 & Secondary & $\leq 1000$ & 44 & 6 & Secondary & $\leq 1000$ \\
11 & 4 & Primary & $\leq 250$ & 45 & 1 & Primary & Any \\
12 & 4 & Secondary & $\leq 250$ & 46 & 1 & Secondary & Any \\
13 & 3 & Primary & $\leq 500$ & 47 & 2 & Primary & Any \\
14 & 3 & Secondary & $\leq 500$ & 48 & 2 & Secondary & Any \\
15 & 5 & Primary & $\leq 250$ & 49 & 3 & Primary & Any \\
16 & 5 & Secondary & $\leq 250$ & 50 & 3 & Secondary & Any \\
17 & 3 & Primary & $\leq 1000$ & 51 & 4 & Primary & $\leq 1500$ \\
18 & 3 & Secondary & $\leq 1000$ & 52 & 4 & Secondary & $\leq 1500$ \\
19 & 6 & Primary & $\leq 250$ & 53 & 5 & Primary & $\leq 1500$ \\
20 & 6 & Secondary & $\leq 250$ & 54 & 5 & Secondary & $\leq 1500$ \\
21 & 1 & Primary & $\leq 1500$ & 55 & 6 & Primary & $\leq 1500$ \\
22 & 1 & Secondary & $\leq 1500$ & 56 & 6 & Secondary & $\leq 1500$ \\
23 & 2 & Primary & $\leq 1500$ & 57 & 4 & Primary & $\leq 2500$ \\
24 & 2 & Secondary & $\leq 1500$ & 58 & 4 & Secondary & $\leq 2500$ \\
25 & 3 & Primary & $\leq 1500$ & 59 & 5 & Primary & $\leq 2500$ \\
26 & 3 & Secondary & $\leq 1500$ & 60 & 5 & Secondary & $\leq 2500$ \\
27 & 4 & Primary & $\leq 500$ & 61 & 6 & Primary & $\leq 2500$ \\
28 & 4 & Secondary & $\leq 500$ & 62 & 6 & Secondary & $\leq 2500$ \\
29 & 5 & Primary & $\leq 500$ & 63 & 4 & Primary & Any \\
30 & 5 & Secondary & $\leq 500$ & 64 & 4 & Secondary & Any \\
31 & 6 & Primary & $\leq 500$ & 65 & 5 & Primary & Any \\
32 & 6 & Secondary & $\leq 500$ & 66 & 5 & Secondary & Any \\
33 & 1 & Primary & $\leq 2500$ & 67 & 6 & Primary & Any \\
34 & 1 & Secondary & $\leq 2500$ & 68 & 6 & Secondary & Any \\
\bottomrule
\end{tabular}
\end{table*}

%% file: text/appendix/appendix-llms.tex
\section{LLM experiments}
\label{app:llm}

We filter the user responses to only those containing at least one query with a natural language explanation. This retains $13/26$ users. We evaluate each method on the validation queries. 

\paragraph{Basic LLM}
The LLM is given the user's elicitation history (past comparisons, their answers, and explanations) and the new pair, and is asked to predict the user's answer. This basic variation receives no information from the linear model. The LLM receives the query as it would be presented to the user, rendered into markdown text. 

\paragraph{LLM with linear heuristic}
This method is identical to the LLM predictor, but the prompt additionally includes the learned linear utility model's prediction and its decision margin. The accompanying prompt instructs the model to treat the linear prediction as a heuristic, and to defer to it when its margin is large and no contrary evidence exists, but to override it when the user's explanations or a clear pattern indicate nonlinear or context-specific structure.

\paragraph{Gated}
Motivated by the fact that at large margins, $\delta$, the linear model is very accurate, we use a switch block that uses the linear model when $\delta \ge 0.1$ and the heuristic LLM otherwise. 

We evaluate three LLM backends: \texttt{gpt-5.4-mini}, \texttt{gpt-6-luna}, and \texttt{mistral} (7B parameters). All backends share the same prompts and parser. Model outputs are retrieved as a single JSON object. Unparseable responses are counted as incorrect answers after five retries.

\subsection{Prompt templates}
\label{app:llm-prompts}

\paragraph{Shared user information}
Each comparison is rendered as a six-line table---one line per
metric---showing both policies' clinical values and a precomputed,
direction-corrected verdict. The history lists every answered elicitation
comparison with the human's choice and explanation; the new pair follows, then the linear-model block, and a fixed instruction.

{\footnotesize
\begin{verbatim}
=== Elicitation History ===

Comparison 1:
  Life Years Gained per Transplant: A=14.0, B=11.5  ->  favors Policy A by 0.15 of range
  Waitlisted Survival Rate within 1 Year: A=82.0%, B=90.0%  ->  favors Policy B by 0.16 of range
  Waitlisted Transplant Rate within 1 Year: A=70.0%, B=70.0%  ->  ~equal
  Transplant Rate, Socioec. Disadvantaged: A=55.0%, B=40.0%  ->  favors Policy A by 0.25 of range
  Transplant Rate, Biol. Disadvantaged: A=48.0%, B=52.0%  ->  favors Policy B by 0.07 of range
  Average Miles per Transplant: A=320 mi, B=180 mi  ->  favors Policy B by 0.14 of range
  Human chose: Policy A
  Explanation: [free-text rationale, if the user provided one]

[... one block per answered elicitation comparison ...]

=== New Comparison ===

  [six advantage lines in the same format, for the held-out pair]

[HYBRID ONLY: the Linear Model Information block is inserted here; see below]

Based on the elicitation history, predict how this human would answer. Respond
only with a JSON object containing "reasoning" and "answer".
\end{verbatim}
}

\paragraph{Linear model information}
For the linear heuristic method, we append the following prompt, immediately before the final instruction.

{\footnotesize
\begin{verbatim}
=== Linear Model Information ===
A linear utility model was fit to this user's earlier comparisons. Its estimated
importance weights (share of total, higher = matters more to this user) are:
  Life Years Gained per Transplant: 22%
  Waitlisted Survival Rate within 1 Year: 18%
  Waitlisted Transplant Rate within 1 Year: 9%
  Transplant Rate, Socioec. Disadvantaged: 25%
  Transplant Rate, Biol. Disadvantaged: 14%
  Average Miles per Transplant: 12%
For THIS comparison the linear model predicts: Policy A, with decision margin
|w*(A-B)| = 0.137 (larger margin = more confident).
\end{verbatim}
}

\paragraph{System prompt for basic LLM method}

{\footnotesize
\begin{verbatim}
You are a preference-learning assistant.

You will be shown a series of pairwise comparisons made by a human stakeholder
evaluating heart transplant allocation policies. Each policy is described by six
clinical metrics. The human has already compared several pairs and, in most
cases, explained their reasoning. Your task is to infer their underlying utility
function -- the relative weight they place on each metric and the tradeoffs they
accept -- and predict how they would answer a NEW comparison.

The six metrics (with the range they vary over) are:
  1. Life Years Gained per Transplant         (years, 5-22;   higher is better)
  2. Waitlisted Survival Rate within 1 Year   (%, 50-100;     higher is better)
  3. Waitlisted Transplant Rate within 1 Year (%, 40-100;     higher is better)
  4. Transplant Rate, Socioeconomically
     Disadvantaged Patients                   (%, 20-80;      higher is better)
  5. Transplant Rate, Biologically
     Disadvantaged (Sensitized) Patients      (%, 20-80;      higher is better)
  6. Average Miles per Transplant             (miles, 0-1000; LOWER is better)

HOW COMPARISONS ARE PRESENTED
Each comparison lists, for every metric, both policies' values and a precomputed
verdict of the form "favors Policy A by 0.24 of range" (or "~equal"). This
verdict is the ground truth for who wins that metric: the arithmetic and
normalization have already been done for you, and the direction of each metric
is already accounted for -- including that FEWER miles is better, so trust the
"favors ..." arrow over the raw numbers for metric 6. "0.24 of range" means the
gap is 24% of that metric's full span, so advantages are directly comparable
across metrics. Do not recompute these; use them.

Reason through the prediction in these steps (think privately first, then answer):

  1. INFER PREFERENCES. From the elicitation history, identify what the human
     values. For each past comparison, read the advantage table: note which
     metrics the chosen policy won and which it gave up, and how the human's
     explanation justified the choice. Look for recurring patterns -- e.g.,
     prioritizing equity (metrics 4-5) over aggregate efficiency (metrics 1-3),
     tolerating longer travel for better survival, treating small advantages as
     ties, or a hard threshold on some metric. Trust stated reasons over guesses.

  2. WEIGH THE NEW PAIR. Read the advantage table for the new comparison. For
     each metric, combine its "of range" magnitude with how much the human has
     shown they care about that metric. Give more influence to metrics they
     prioritize and discount ones they appeared indifferent to.

  3. DECIDE. Choose the policy the human would most likely prefer. Answer 0 (tie)
     only if the net weighted advantage is genuinely negligible under their
     preferences, or if their past behavior shows they'd call a difference this
     small a tie -- not merely because the comparison is close.

If the elicitation history is empty, fall back to a reasonable default: weigh the
metrics roughly equally, summing each metric's signed "of range" advantage.

Respond with a single JSON object on one line -- no other text.
The "answer" field must be an integer (not a string): 1, -1, or 0.
  {"reasoning": "<brief explanation>", "answer": 1}   -- Policy A is better
  {"reasoning": "<brief explanation>", "answer": -1}  -- Policy B is better
  {"reasoning": "<brief explanation>", "answer": 0}   -- they are equivalent
\end{verbatim}
}

\paragraph{System prompt for linear heuristic LLM method}

{\footnotesize
\begin{verbatim}
[... identical six-metric preamble and "HOW COMPARISONS ARE PRESENTED"
     section as the LLM method ...]

YOU ARE ALSO GIVEN A LINEAR MODEL.
A "Linear Model Information" section provides a linear utility model that was fit
to this user's earlier comparisons: its per-metric importance weights, plus its
prediction and a decision margin for the current pair. Treat this as a STRONG
PRIOR, not as the answer:
  - The linear model captures the user's AVERAGE linear tradeoffs well, so when
    its margin is large it is usually right -- agree unless you have a specific
    reason not to.
  - But it cannot capture nonlinear or context-specific preferences: thresholds,
    diminishing returns, lexicographic priorities, or principles the user stated
    in their explanations. When the user's explanations or a clear pattern in the
    history contradict the linear model -- especially when its margin is small --
    trust the human evidence and override it.

Reason through the prediction in these steps (think privately first, then answer):

  1. INFER PREFERENCES. [...as LLM method, plus: look for nonlinear structure
     such as thresholds, hard priorities, tradeoffs they refuse...]

  2. WEIGH THE NEW PAIR. [...as LLM method...]

  3. RECONCILE WITH THE LINEAR MODEL. Compare your read against the linear
     model's prediction and margin. If they agree, that reinforces the answer. If
     they disagree, decide whether the discrepancy is explained by nonlinear
     structure or stated principles (override) or is just noise (defer to the
     linear model, especially at large margin).

  4. DECIDE. [...as LLM method...]

Respond with a single JSON object on one line -- no other text.
The "answer" field must be an integer (not a string): 1, -1, or 0.
\end{verbatim}
}

%% file: text/appendix/irb.tex
\section{IRB disclosure}
\label{sec:appendix:irb}

This study was deemed exempt by the Institutional Review Board(s) (IRB) of the home institution(s). Informed consent was obtained from all participants prior to data collection. All data were anonymized and securely stored to protect the privacy and confidentiality of participants.